\documentclass{article}

\PassOptionsToPackage{numbers, compress}{natbib}

    \usepackage[preprint]{neurips_2022}

\usepackage[utf8]{inputenc} 
\usepackage[T1]{fontenc}    
\usepackage[pagebackref=true]{hyperref}       
\usepackage{url}            
\usepackage{booktabs}       
\usepackage{amsfonts}       
\usepackage{nicefrac}       
\usepackage{microtype}      
\usepackage{xcolor}         
\usepackage{pifont}

\renewcommand*\backref[1]{\ifx#1\relax \else (Cited on #1) \fi}  

\title{Bilinear Exponential Family of MDPs: Frequentist Regret Bound with Tractable Exploration \& Planning}

\author{%
  Reda Ouhamma\thanks{https://redaouhamma.github.io/} \qquad Debabrota Basu \qquad Odalric-Ambrym Maillard \\
  Univ. Lille, Inria, CNRS,\\
  Centrale Lille, UMR 9189 CRIStAL,\\
  F-59000 Lille, France\\
  \texttt{reda.ouhamma@gmail.com , debabrota.basu@inria.fr , odalric.maillard@inria.fr} \\
}

\usepackage{amssymb,amsfonts,amsthm, bbm,enumitem,xspace, todonotes, mathtools}
\usepackage{algorithm,algorithmic}
\usepackage{longtable,tcolorbox}

\usepackage[nohints]{minitoc}

\newtheorem{theorem}{Theorem}
\newtheorem{lemma}[theorem]{Lemma}
\newtheorem{corollary}[theorem]{Corollary}
\newtheorem{assumption}[theorem]{Assumption}
\newtheorem{remark}{Remark}

\newcommand \dd {\,\mathrm{d}}
\newcommand{\bR}{\mathbb{R}}
\newcommand{\bN}{\mathbb{N}}
\newcommand{\bE}{\mathbb{E}}
\newcommand{\bP}{\mathbb{P}}
\newcommand{\bA}{\mathbb{A}}
\newcommand{\bS}{\mathbb{S}} 

\newcommand{\cS}{\mathcal{S}}
\newcommand{\cR}{\mathcal{R}}
\newcommand{\cA}{\mathcal{A}} 
\newcommand{\cM}{\mathcal{M}} 

\newcommand{\cN}{\mathcal{N}} 

\newcommand{\V}[1]{V_{\thetap,\thetar,#1}}

\newcommand{\tV}[1]{V_{\thetahp,\thetatr,#1}}
\newcommand{\tQ}[1]{Q_{\thetahp,\thetatr,#1}}
\newcommand{\pG}{\bar{G}_{k}^\texttt{p}}
\newcommand{\rG}{\bar{G}_{k}^\texttt{r}}

\newcommand{\thetar}{\theta^{\texttt{r}}}
\newcommand{\thetarp}{\theta^{\texttt{r} \prime}}
\newcommand{\thetap}{\theta^{\texttt{p}}}
\newcommand{\thetapp}{\theta^{\texttt{p} \prime}}
\newcommand{\thetahr}{\hat{\theta}^{\texttt{r}}}
\newcommand{\thetatr}{\tilde{\theta}^{\texttt{r}}}

\newcommand{\thetahp}{\hat{\theta}^{\texttt{p}}}

\newcommand{\probr}[1]{\mathbb{P}_{\thetar}^\texttt{r}\left(#1\right)}
\newcommand{\probrp}[1]{\mathbb{P}_{\theta^{\prime \texttt{r}}}^\texttt{r}\left(#1\right)}

\newcommand{\betar}{\beta^{\texttt{r}}}
\newcommand{\betap}{\beta^{\texttt{p}}}
\newcommand{\alphar}{\alpha^{\texttt{r}}}
\newcommand{\alphap}{\alpha^{\texttt{p}}}

\newcommand{\ppsi}{\psi^{\texttt{p}}}
\newcommand{\rz}{Z^{\texttt{r}}}
\newcommand{\pz}{Z^{\texttt{p}}}
\newcommand{\pmu}{\mu^{\texttt{p}}}

\newcommand{\pphi}{\phi^{\texttt{p}}}

\newcommand{\Var}{\mathbb{V}\mathrm{ar}}

\newcommand{\prob}[1]{\mathbb{P}\left(#1\right)}

\renewcommand{\det}{\operatorname{det}}
\newcommand{\tr}{\operatorname{tr}}

\DeclareMathOperator*{\argmax}{arg\,max}
\DeclareMathOperator*{\argmin}{arg\,min}

\newcommand{\eqdef}{\stackrel{\rm def}{=}}

\newcommand{\ie}{\textit{i.e.} }
\newcommand{\eg}{\textit{e.g.} }
\newcommand{\cf}{\textit{cf.} }
\newcommand{\algo}{\texttt{BEF-RLSVI}\xspace}
\newcommand{\RLSVI}{\texttt{RLSVI}\xspace}

\usepackage{mathrsfs}
\DeclareRobustCommand{\bigO}{%
  \text{\usefont{OMS}{cmsy}{m}{n}O}%
}

\begin{document}

\maketitle

\setcounter{parttocdepth}{4}
\doparttoc 
\faketableofcontents 

\begin{abstract}
We study the problem of episodic reinforcement learning in continuous state-action spaces with unknown rewards and transitions. Specifically, we consider the setting where the rewards and transitions are modeled using parametric bilinear exponential families. We propose an algorithm, \algo, that a) uses penalized maximum likelihood estimators to learn the unknown parameters, b) injects a calibrated Gaussian noise in the parameter of rewards to ensure exploration, and c) leverages linearity of the exponential family with respect to an underlying RKHS to perform tractable planning. We further provide a frequentist regret analysis of \algo that yields an upper bound of $\tilde{\bigO}(\sqrt{d^3H^3K})$, where $d$ is the dimension of the parameters, $H$ is the episode length, and $K$ is the number of episodes. Our analysis improves the existing bounds for the bilinear exponential family of MDPs by $\sqrt{H}$ and removes the handcrafted clipping deployed in existing \RLSVI-type algorithms. Our regret bound is order-optimal with respect to $H$ and $K$.

\end{abstract}

\section{Introduction}
Reinforcement Learning (RL) is a well-studied and popular framework for sequential decision making, where an agent aims to compute a \textit{policy} that allows her to maximize the accumulated reward over a horizon by interacting with an \textit{unknown} environment~\citep{sutton2018reinforcement}.

\noindent\textbf{Episodic RL.} In this paper, we consider the episodic finite-horizon MDP formulation of RL, in short \textit{Episodic RL}~\cite{osband2013more, azar2017minimax,dann2017unifying}. Episodic RL is a tuple $\cM = \langle \cS, \cA, \bP, r, K, H \rangle$, where the state (resp. action) space $\cS$ (resp. $\cA$) might be continuous. In episodic RL, the agent interacts with the environment in episodes consisting of $H$ steps. Episode $k$ starts by observing state $s_1^k$. Then, for $t=1,\ldots H$, the agent draws action $a_t^k$ from a (possibly time-dependent) policy $\pi_t(s_t^k)$, observes the reward $r(s_t^k,a_t^k) \in [0,1]$, and transits to a state $s_{t+1}^k \sim \bP(.\mid s_t^k,a_t^k)$ according to the transition function $\bP$. The performance of a policy $\pi$ is measured by the total expected reward $V_{1}^\pi$ starting from a state $s\in \cS$, the value function and the state-action value functions at step $h\in [H]$ are defined as
\begin{equation*}
    V_h^\pi (s) \eqdef \bE \left[\sum_{t=h}^H r(s_t,a_t) \mid s_h =s \right], \quad\text{and}\quad    Q_h^\pi (s,a) \eqdef \bE \left[\sum_{t=h}^H r(s_t,a_t) \mid s_h =s, a_h=a \right].
\end{equation*}
Here, computing the policy leading to maximization of cumulative reward requires the agent to strategically control the actions in order to learn the transition functions and reward functions as precisely as required. This tension between learning the unknown environment and reward maximization is quantified as \textit{regret}: the typical performance measure of an episodic RL algorithm. \textit{Regret} is defined as the difference between the \textit{expected cumulative reward} or \textit{value} collected by the optimal agent that knows the environment and the expected cumulative reward or value obtained by an agent that has to learn about the unknown environment. Formally, the regret over $K$ episodes is $$\cR(K) \triangleq \sum_{k=1}^K \left(V_{1}^{\pi^\star}(s_1^k)-V_{1}^{\pi_t}(s_1^k) \right).$$



\noindent\textbf{Key Challenges.} \textit{The first key challenge in episodic RL is to tackle the exploration--exploitation trade-off}. This is traditionally addressed with the \textit{optimism principle} that either carefully crafts optimistic upper bounds on the value (or state-action value) functions~\cite{azar2017minimax}, or maintains a posterior on the parameters to perform posterior sampling~\cite{osband2013more}, or perturbs the value (or state-action value) function estimates with calibrated noise~\cite{osband2016generalization}. Though the first two approaches induce theoretically optimal exploration, they might not yield tractable algorithms for large/continuous state-action spaces as they either involve optimization in the optimistic set or maintaining a high-dimensional posterior. Thus, \textit{we focus on extending the third approach of \emph{Randomized Least-Square Value Iteration (RLSVI)} framework, and inject noise only in rewards to perform tractable exploration.}

\textit{The second challenge}, which emerges \textit{for continuous state-action spaces}, \textit{is to learn a parametric functional approximation of either the value function or the rewards and transitions} in order to perform planning and exploration. Different functional representations (or models), such as linear~\cite{jin2020provably}, bilinear~\cite{du2021bilinear}, and bilinear exponential families~\cite{chowdhury2021reinforcement}, are studied in literature to develop optimal algorithms for episodic RL with continuous state-action spaces. Since the linear assumption is restrictive in real-life -where non-linear structures are abundant-, generalized representations have obtained more attention recently~\cite{chowdhury2021reinforcement, li2021exponential, du2021bilinear, foster2021statistical}. The bilinear exponential family model is of special interest as it is expressive enough to represent tabular MDPs (discrete state-action), factored MDPs~\cite{kearns1999efficient} and linearly controlled dynamical systems (such as Linear Quadratic Regulators~\cite{abbasi2011regret}) as special cases~\cite{chowdhury2021reinforcement}. Thus, in this paper, \textit{we study the bilinear exponential family of MDPs, i.e. the episodic RL setting where the rewards and transition functions can be modelled with bilinear exponential families.}

\textit{The third challenge is to perform tractable planning\footnote{By tractable planning, we mean having a planner with (pseudo-)polynomial complexity in the problem parameters, i.e. dimension of parameters, dimension of features, horizon, and number of episodes.} given the perturbation for exploration and the model class.} Existing work~\citep{osband2014model, chowdhury2021reinforcement} assumes an oracle to perform planning and yield policies that aren't explicit. The main difficulty in such planning approaches is that dynamic programming requires calculating $\int \bP(s'\mid s,a) V_{h}(s)$ for all $(s,a)$ pairs. This is not trivial unless the transition is assumed to be linear and decouples $s'$ from $(s,a)$, which is not known to hold except for tabular MDPs. Much ink has been spilled about this challenge recently, \eg \cite{du2019good} asks when misspecified linear representations are enough for a polynomial sample complexity in several settings. \cite{shariff2020efficient, lattimore2020learning, van2019comments} provide positive answers for specific linear settings. In this paper, \textit{we aim to design a tractable planner for the bilinear exponential family representation.}

In this paper, we aim to address the following question that encompasses the three challenges:
\begin{center}
    Can we design an algorithm that performs \textbf{tractable exploration} and \textbf{planning} for \textit{bilinear exponential family of MDPs} yielding a  \textbf{near-optimal frequentist regret bound}?
\end{center}

\noindent\textbf{Our Contributions.} Our contributions to this question are three-fold.

1. \textit{Formalism:} We assume that neither rewards nor transitions are known, whereas existing efforts on the bilinear exponential family of MDPs assume knowledge of rewards. This makes the addressed problem harder, practical, and more general. We also observe that though the transition model can represent non-linear dynamics, it implies a linear behavior (see Section~\ref{sec:Bilenar_exponential_family}) in a Reproducible Kernel Hilbert Space (RKHS). This observation contributes to the tractability of planning.

2. \textit{Algorithm:} We propose an algorithm \algo that extends the \RLSVI framework to bilinear exponential families (see Section~\ref{sec:algo_description}). \algo a) injects calibrated Gaussian noise in the rewards to perform exploration, b) leverages the linearity of the transition model with respect to an underlying RKHS to perform tractable planning and c) uses penalized maximum likelihood estimators to learn the parameters corresponding to rewards and transitions (see Section~\ref{sec:building_blocks}). To the best of our knowledge, \textit{\algo is the first algorithm for the bilinear exponential family of MDPs with tractable exploration and planning under unknown rewards and transitions.}

3. \textit{Analysis:} We carefully develop an analysis of \algo that yields $\tilde{\bigO}(\sqrt{d^3 H^3 K})$ regret which improves the existing regret bound for bilinear exponential family of MDPs with known reward by a factor of $\sqrt{H}$ (Section~\ref{sec:upper_bound}). Our analysis (Section~\ref{sec:proof_outline}) builds on existing analyses of RLSVI-type algorithms~\cite{osband2016generalization}, but contrary to them, we remove the need to handcraft a clipping of the value functions~\cite{zanette2020frequentist}. We also do not need to \emph{assume} anti-concentration bounds as we can explicitly control it by the injected noise. This was not done previously except for the linear MDPs. We illustrate this comparison in Table~\ref{tab:comparison_RL_algos}. We highlight three technical tools that we used to improve the previous analyses: 1) Using transportation inequalities instead of the simulation lemma reduces a $\sqrt{H}$ factor compared to \cite{ren2021free}, 2) Leveraging the observation that true value functions are bounded enables using an improved elliptical lemma (compared to \cite{chowdhury2021reinforcement}), and 3) Noticing that the norm of features can only be large for a finite amount of time allows us to forgo clipping and reduce a $\sqrt{d}$ factor from the regret compared to~\cite{zanette2020frequentist}.

\setlength{\textfloatsep}{10pt}
\begin{table*}
   \caption{A comparison of RL Algorithms for MDPs with functional representations.}\label{tab:comparison_RL_algos}
    \centering
    \resizebox{\textwidth}{!}{
    \begin{tabular}{cccccc}
    \toprule
    Algorithm  &  Regret & Tractable & Tractable & Free of & Model, assumptions \\
    && exploration & planning & clipping\\
    \midrule
    Thompson sampling & {\footnotesize$\sqrt{d^2 H^3 K}$} & \ding{55} \quad & \quad \ding{51} & N.A & Gaussian $\bP$\\
    \cite{ren2021free}& (Bayesian) &  & & &Known rewards \\
    \midrule
    \texttt{EXP-UCRL} & {\footnotesize$\sqrt{d^2 H^4 K}$} & \ding{55} \quad & \quad \ding{55}  & N.A & Bilinear Exp Family (BEF)\\
    \cite{chowdhury2021reinforcement}& (Frequentist)&  & & &known rewards \\
    \midrule
    \texttt{SMRL} \cite{li2021exponential} & {\footnotesize$\sqrt{d^2 H^4 K}$} & \ding{55} \quad & \quad \ding{55}  & N.A & BEF, known rewards\\
    \midrule
    \texttt{UCRL-VTR} \cite{pmlr-v119-ayoub20a} & {\footnotesize$\sqrt{d^2 H^4 K}$} & \ding{55} \quad & \quad \ding{55}  & N.A & Linear mixture model\\
    \midrule
    $\mathcal{F}-$\texttt{PHE-LSVI} \cite{ishfaq2021randomized} & {\footnotesize$\operatorname{poly}(d_E H)\sqrt{K H}$} & \ding{51} \quad & \quad \ding{55}  & \ding{55} & Eluder dimension, Tabular \\
    $\texttt{PHE-LSVI}$ (linear-RL) & $\sqrt{d^3 H^4 K}$ &&&& Anti-concentration \\
    \midrule
    \texttt{UC-MatrixRL} \cite{yang2020reinforcement} & {\footnotesize$\sqrt{d^2 H^5 K}$} & \ding{55} \quad & \quad \ding{55}  & N.A & Linear factor MDP \\
    \midrule
    \texttt{OPT-RLSVI} \cite{zanette2020frequentist} & {\footnotesize$\sqrt{d^4 H^5 K}$} & \ding{51} \quad & \quad \ding{51}  & \ding{55} & Linear $V$ \\
    \midrule
    \algo (this work) & {\footnotesize$\sqrt{d^3 H^3 K}$} & \ding{51} \quad & \quad \ding{51}  & \ding{51} & Bilinear Exp Family \\
    \bottomrule
    \end{tabular}}
\end{table*}

\section{Bilinear exponential family of MDPs}\label{sec:Bilenar_exponential_family}

In this section, we introduce the bilinear exponential family (BEF) model coined in~\cite{chowdhury2021reinforcement}, extend it to parametric rewards, and we state a novel observation about linearity of this representation. 

\noindent\textbf{Bilinear exponential family.} We consider transitions and rewards from the BEF. Specifically,
\begin{align}
    \label{def:transition_model}
    \prob{\tilde{s} \mid s,a} &= \exp\left(\psi (\tilde{s})^\top M_{\thetap}\varphi(s,a) - Z_{s,a}^\texttt{p}(\thetap)\right),\\
    \prob{r \mid s,a} &= \exp\left(r\: B^\top M_{\thetar}\varphi(s,a) - Z_{s,a}^\texttt{r}(\thetar)\right), \label{def:reward_model}
\end{align}
where $\varphi \in (\bR_{+}^{q})^{\cS\times \cA}$ and $\psi \in (\bR_{+}^{p})^{\cS}$ are known feature mappings, and $B \in \bR^p$ is a known matrix. The reward and transition parameters are $\thetap, \thetar \in \bR^d$. $M_{\theta^{\cdot}} \eqdef \sum_{i=1}^d \theta_i^{\cdot} A_i $, where $(A_i)_{1\le i\le d}$ are known matrices. The log partition function: $\quad \pz_{s,a}(\thetap) \eqdef \log \int_{\cS} \exp\left(\psi(\tilde{s})^\top M_{\thetap} \varphi(s,a)\right)d\tilde{s}, \quad$ and $Z^{\texttt{r}}$ is defined similarly. Finally, we emphasize a minor difference with the original BEF model: like~\cite{li2021exponential}, we omit a base measure of the form $h(s,\tilde{s},a)$ from the model, note that all the examples provided in \cite{chowdhury2021reinforcement} still hold with this slight restriction. 


We denote $V_{\thetap,\thetar,h}^\pi$, respectively $Q_{\thetap,\thetar,h}^\pi$, the value, respectively state-action value function for policy $\pi$ in the MDP parameterized by $(\thetap,\thetar)$ at time $h$. A policy $\pi^\star$ is \textit{optimal} if for all $s\in \cS,\: V_{\theta,h}^{\pi^\star}(s) = \max\limits_{\pi \in \Pi} V_{\theta,h}^{\pi}(s)$. A learning algorithm minimizes the (pseudo-)regret defined as:
\begin{equation}
    \label{def:regret}
    \cR(K) \triangleq \sum_{k=1}^K \left(V_{\theta,1}^{\pi^\star}(s_1^k)-V_{\theta,1}^{\pi^t}(s_1^k) \right).
\end{equation}

\noindent\textbf{Linearity of transitions.}
Now, we state an observation about the bilinear exponential family and discuss how it helps with the challenge of planning in episodic RL. Specifically, the popular assumption of linearity of the transition kernel is a direct consequence of our model. Indeed,
\begin{equation*}
   2 \psi\left(s^{\prime}\right)^{\top} M_{\thetap} \varphi(s, a) = -\|(\psi(s^{\prime}) - M_{\thetap}\varphi(s, a)\|^2
    + \|\psi(s')\|^2+\|M_{\thetap}\varphi(s,a)\|^2.
\end{equation*}
Notice that the quadratic term resembles the Radial Basis Function (RBF) kernel. More precisely, for an RBF kernel with covariance $\Sigma \!=\! I_p$ and $k(x,y) \!\eqdef\! \exp\left(-\|x-y\|^2 /2\right)$, we find

\begin{equation}\label{eq:linear_transition}
    \prob{s^{\prime} \mid s, a} = \langle \pphi(s,a), \pmu(s') \rangle_{\mathcal{H}},
\end{equation}

where $\mathcal{H}$ is the RKHS associated with the kernel, $\pmu(s') = (2\pi)^{-p/2} \: k\left(\psi(s'),.\right) \: \exp\left(\|\psi\left(s^{\prime}\right)\|^2 /2\right)$, and $\pphi(s,a) =  k\left(M_{\thetap}^\top \varphi(s,a),.\right) \: \exp\left(\|M_{\thetap} \varphi(s, a) \|^2 /2  -Z_{s,a}(\thetap) \right)$. Equation~\eqref{eq:linear_transition} shows that $s'$ is decoupled from $(s,a)$, we see hereafter why this is crucial to reducing the complexity of planning. 
\begin{remark}
Up to our knowledge, \cite{ren2021free} is the only work providing an example of linear transition kernel for RL with continuous state-action spaces. They consider Gaussian transitions with an unknown mean ($f^\star (s,a)$) and known variance ($\sigma^2$). Actually, linear $f^\star$ is a special case of the bilinear exponential family model, where $\psi(s')= (s', \|s'\|^2)$ and $M_\theta \varphi(s,a) = (f_\theta (s,a)/\sigma^2 , -1/\sigma^2)$.
\end{remark}


\noindent\textbf{Importance of linearity.}
To understand the planning challenge in RL, recall the Bellman equation:
\begin{equation*}
    Q_h^\pi(s,a) = r(s,a) + \int_{\tilde{s}\in \cS} P(s'\mid s,a) V_{h+1}^\pi(\tilde{s}) d \tilde{s},
\end{equation*}
We must approximate the integral at the R.H.S.for $(s,a) \in \cS \times \cA$. For a tabular MDP with $|S|$ states and $|A|$ actions, we need to evaluate $(Q_h^\pi)_{h\in [H]}$, i.e. to approximate $|S| \times |A| \times H$ integrals per episode, which can be very expensive. However, with the linear transition model of Equation~\eqref{eq:linear_transition}, although $\pphi$ and $\pmu$ are infinite dimensional, we show in Section~\ref{sec:building_blocks} (§~planning) that the planning complexity becomes polynomial in the problem parameters.


\section{\algo: algorithm design and frequentist regret bound}\label{sec:algo_description}
In this section, we formally introduce the Bilinear Exponential Family Randomized Least-Squares Value Iteration (\algo) algorithm along with a high probability upper-bound on its regret.

\subsection{\algo: algorithm design}

\algo is based on \RLSVI~\cite{osband2016generalization} framework with the distinction that we only perturb the reward parameters and not all the parameters of the value function. \RLSVI algorithms are reminiscent of Thompson Sampling, yet more tractable with better control over the probability to be optimistic.

\begin{algorithm}[!h]
	\caption{\algo}
	\label{algo:EXP-RLSVI}
	\begin{algorithmic}[1]
		\STATE \textbf{Input:} failure rate $\delta$, constants $\alphap,\eta$ and $(x_k)_{k\in[K]}\in \bR^+$
		\FOR{episode $k=1,2,\ldots$}
		\STATE Observe initial state $s_{1}^k$
		\STATE Sample noise $\xi_k \sim \cN\left(0,x_k (\pG)^{-1}\right)$ such that\label{alg_1:line_4} 
		\begin{center}$\pG = \frac{\eta}{\alphap}\bA+ \sum_{\tau=1}^{k-1} \sum_{h=1}^H (\varphi(s_h^\tau,a_h^\tau)^\top A_i^\top A_j \varphi(s_h^\tau,a_h^\tau))_{i,j \in [d]}$\end{center}
		\STATE Perturb reward parameter: $\thetatr(k) = \thetahr(k) + \xi_k$
		\STATE Compute $(\tQ{h}^{k})_{h\in [H]}$ via Bellman-backtracking, see Algorithm~\ref{algo:Bellman}
		\FOR{$h=1,\ldots,H$}
		\STATE Pull action $a_h^k=\argmax_{a} \tQ{h} (s_h^k,a)$
		\STATE Observe reward $r(s_h^k,a_h^k)$ and state $s_{h+1}^k$. 
		\ENDFOR
		\STATE Update the penalized ML estimators $\thetahp(k), \thetahr(k)$, see Equation~\eqref{eq:ML_transition} and Equation~\eqref{eq:ML_reward}
		\ENDFOR
	\end{algorithmic}
\end{algorithm}
We can see that Algorithm~\ref{algo:EXP-RLSVI} performs exploration by a Gaussian perturbation of the reward parameter (Line~\ref{alg_1:line_4}). Contrary to optimistic approaches, this method is explicit and also more efficient since it does not a involve high-dimensional optimization.

\begin{algorithm}
    \caption{Bellman Backtracking}
    \label{algo:Bellman}
    \begin{algorithmic}[1]
        \STATE \textbf{Input} Parameters $\thetahp, \thetatr$, initialize $ \tilde{\theta} = (\thetatr, \thetahp)$ and $\forall s, V_{H+1}(s) = 0$
        \FOR{steps $h=H-1, H-2, \cdots, 0$}
        \STATE Calculate $Q_{\tilde{\theta},h}(s, a) = \bE_{s,a}^{\thetatr}[r] + \langle \pphi(s, a), \int V_{ \tilde{\theta},h+1}(s^\prime) \pmu(s^\prime) d s^\prime\rangle_{\mathcal{H}}$.\label{alg_2:planning}
        \ENDFOR
    \end{algorithmic}
\end{algorithm}
We can approximate Line~\ref{alg_2:planning} of Algorithm~\ref{algo:Bellman} with $\bigO(p H^3 K\log(H K))$ complexity and without harming the learning process (\cf §~planning, Section~\ref{sec:building_blocks}). Therefore, here, planning is tractable.

\subsection{\algo: regret upper-bound} \label{sec:upper_bound}

We state the standard smoothness assumptions on the model~\citep{chowdhury2021reinforcement,jun2017scalable,lu2021low}.
\begin{assumption}
\label{ass:bounds_hessians}
There exist constants $\alphap, \alphar,\betap,\betar >0$, such that the representation model satisfies:
\begin{align*}
    \forall (s,a) \in\cS\times \cA, \forall \theta, x \in \bR^d \quad &\alphap \le x^\top C_{s,a}^\theta [\psi] x \le \betap\\
    \forall (s,a)\in\cS\times \cA, \forall \theta, x \in \bR^d \quad &\alphar \le \Var_{s,a}^\theta (r) \:x^\top B^\top B x \le \betar
\end{align*}
where $\mathbb{C}_{s, a}^{\theta}\left[\psi\left(s^{\prime}\right)\right] \triangleq \mathbb{E}_{s'\sim \bP_{\theta}\mid s, a} \left[\psi\left(s^{\prime}\right) \psi\left(s^{\prime}\right)^{\top}\right]-\mathbb{E}_{s'\sim \bP_{\theta}\mid s, a}\left[\psi\left(s^{\prime}\right)\right] \mathbb{E}_{s'\sim \bP_{\theta}\mid s, a} \left[\psi\left(s^{\prime}\right)^{\top}\right]$ and $\Var_{s,a}^\theta (r) \triangleq \left(\mathbb{E}_{s, a}^{\theta}\left[r^2\right]-\mathbb{E}_{s, a}^{\theta}\left[r\right]^2\right)$ is the variance of the reward under $\theta$.
\end{assumption}
A closer look at the derivatives of the model (see Appendix~\ref{app:properties_exp_fam}) tells us that previous inequalities directly imply a control over the eigenvalues of the Hessian matrices of the log-normalizers.

We now state our main result, the regret upper-bound of \algo.

\begin{theorem}[Regret bound]\label{thm:regret_bound}
Let $\bA \triangleq (\tr(A_i A_j^\top))_{i,j \in [d]}$ and $G_{s,a} \triangleq (\varphi(s,a)^\top A_i^\top A_j \varphi(s,a))_{i,j \in [d]}$. Under Assumption~\ref{ass:bounds_hessians} and further considering that
\begin{enumerate}[nosep,leftmargin=*]
    \item $\max\{\|\thetar\|_{\bA},\|\thetap\|_{\bA}\} \le B_{\bA}, \: \:\|\bA^{-1}G_{s,a}\| \le B_{\varphi,\bA}\:$ and $\:\bE_{\thetar}[r(s,a)]\in [0,1]\:$ for all $(s,a)$.
    \item noise $\xi_k \sim \cN(0,x_k (\pG)^{-1})\:$ satisfies $\:x_k \ge \left(H\sqrt{\frac{\betap \betap(K,\delta)}{\alphap\alphar}}+\frac{\sqrt{\betar\betar(K,\delta)\min\{1,\frac{\alphap}{\alphar}\}}}{2 \alphar}\right)^2 \propto d H^2$,
\end{enumerate}
then for all $\delta \in (0,1]$, with probability at least $1-7\delta$,
\begin{align*}
    \mathcal{R}(K) &\le \sqrt{K H}  \Bigg[\underbrace{2 H \left(\sqrt{\frac{2\betap}{\alphap} \betap(K,\delta) \gamma_{K}^\texttt{p}}\! +\! (1\!+\!\sqrt{\gamma^\texttt{r}_K})\sqrt{\log(1/\delta^2)}\right)}_{\text{Transition concentration} ~\approx~ d H} + \underbrace{\betar\sqrt{\frac{\betar(n,\delta)\gamma^\texttt{r}_K}{2\alphar}}}_{\text{Reward concentration}~\approx~d}\\
    &\qquad\qquad+ \underbrace{c\betar\sqrt{x_K d \gamma^\texttt{r}_K\log(d K /\delta)} + \frac{\betar\sqrt{x_K d \gamma^\texttt{r}_K \log(e/\delta^2)}}{\Phi(-1)}(1\!+\!\sqrt{\log(d/\delta)})}_{\text{Noise concentration}~\approx~ d^{3/2} H} \Bigg] \\
    &+ \sqrt{H \gamma^\texttt{r}_K} \Bigg[\underbrace{\betar C_d\left(\!\sqrt{\frac{\betar(K,\delta)}{2\alphar}}\!+c\sqrt{x_K d\log(d K /\delta)}\right) }_{\text{Estimation error for no clipping}~\approx~ d H}\\
    &\qquad\qquad
    + \underbrace{\!\frac{\betar d\sqrt{x_K}}{\Phi(-1)}(1\!+\!\sqrt{\log(d/\delta)}) \sqrt{C_d \left(\!1\!+\!\frac{\alphar B_{\varphi, A} H}{\eta}\!\right)}}_{\text{Learning error for no clipping}~ \approx ~(d H)^{3/2}} \!\Bigg]\,,  
\end{align*}
where for $\texttt{i}\in [\texttt{p}, \texttt{r}]$, $\beta^\texttt{i}(K,\delta)\triangleq \frac{\eta}{2}B_{\bA}^2+\gamma_K^\texttt{i}+\log(1/\delta)$, and $\gamma_K^\texttt{i}\triangleq d\log(1+\frac{\beta^\texttt{i}}{\eta}B_{\varphi,\bA}H K)$. Also, $C_d \triangleq\frac{3d}{\log(2)}\log\left(1+ \frac{\alphar\|\bA\|_2^2 B_{\varphi,\bA}^2}{\eta \log(2)}\right) $, $\Phi$ is the Gaussian CDF, and $c$ is a universal constant.

\end{theorem}

    


Theorem~\ref{thm:regret_bound} entails a regret $\mathcal{R}(K) = \bigO(\sqrt{d^3 H^3 K})$ for \algo, where $d$ is the number of parameters of the bilinear exponential family model, $K$ is the number of episodes, and $H$ is the horizon of an episode. We now clarify how this contrasts with related literature.

\textit{Comparison with other bounds.} The closest work to ours is~\cite{chowdhury2021reinforcement} as it considers the same model for transitions but with known rewards. They propose a \texttt{UCRL}-type and \texttt{PSRL}-type algorithm, which achieve a regret of order $\widetilde{O}(\sqrt{d^2 H^4 K})$. There are two notable algorithmic differences with our work. First, they do exploration using intractable-optimistic upper bounds or high-dimensional posteriors, while we do it with explicit perturbation. The second difference is in planning. While they assume access to a planning oracle, we do it explicitly with pseudo-polynomial complexity (Section~\ref{sec:planning}). Moreover, we improve the regret bound by a $\sqrt{H}$ factor thanks to an improved analysis, (\cf Lemma~\ref{lem:elliptical}). But similar to all \RLSVI-type algorithms, we pick up an extra $\sqrt{d}$ (\cf~\citep{abeille2017linear}).

\cite{zanette2020frequentist} proposes a variant of \RLSVI for continuous state-action spaces, where there are low-rank models of transitions and rewards. They show a regret bound $R(K) = \widetilde{O}(\sqrt{d^4 H^5 K})$, which is larger than that of \algo by $O(\sqrt{d H^2})$. In algorithm design, we improve on their work by removing the need to carefully clip the value function. Analytically, our model allows us to use transportation inequalities (\cf Lemma~\ref{lem:transportation}) instead of the simulation lemma, which saves us a $\sqrt{H}$ factor.

\cite{ren2021free} considers Gaussian transitions, i.e. $s^{\prime}=f^{*}(s, a)+\epsilon$ such that $\epsilon \sim \mathcal{N}\left(0, \sigma^{2}\right)$. This is a particular case of our model. They propose to use Thompson Sampling, and have the merit of being the first to have observed linearity of the value function from this transition structure. But they do not connect it to the finite dimensional approximation of~\cite{rahimi2007random} unlike us (Section~\ref{sec:building_blocks}). Finally, they show a Bayesian regret bound of $O(\sqrt{d^2 H^3 K})$. This notion of regret is weaker than frequentist regret, hence this result is not directly comparable with Theorem~\ref{thm:regret_bound}.

\textit{Tightness of regret bound.} A lower bound for episodic RL with continuous state-action spaces is still missing. However, for tabular RL, \citep{domingues2021episodic} proves a lower bound of order $\Omega(\sqrt{H^3 S A K})$. If we represent a tabular MDP in our model, we would need $d = S^2 \times A$ parameters (Section 4.3, \citep{chowdhury2021reinforcement}). In this case, our bound becomes $R(K) = O(\sqrt{(S^2 A)^3 H^3 K})$, which is clearly not tight is $S$ and $A$. This is understandable due to the relative generality of our setting. We are however positively surprised that \textbf{our bound is tight in terms of its dependence on $H$ and $K$.}



\section{Algorithm design: building blocks of \algo}\label{sec:building_blocks}

We present necessary details about \algo and discuss the key algorithm design techniques.

\noindent\textbf{Estimation of parameters.}\label{sec:estimation} We estimate transitions and rewards from observations similar to \texttt{EXP-UCRL}~\cite{chowdhury2021reinforcement}, \ie by using a penalized maximum likelihood estimator 
\begin{equation*}
    \thetahp(k) \in \argmin _{\theta \in \bR^{d}} \sum_{t=1}^{k} \sum_{h=1}^H -\log \bP_{\theta}\left(s_{h+1}^t \mid s_{h}^t, a_{h}^t\right)+\eta \operatorname{pen} (\theta).
\end{equation*}
Here, pen$(\theta)$ is a trace-norm penalty: $\operatorname{pen}(\theta) = \frac{1}{2}\|\theta\|_{\mathbb{A}}$ and $\bA=(\operatorname{tr}(A_i A_j^\top))_{i,j}$. By properties of the exponential family, the penalized maximum likelihood estimator verifies, for all $i\le d$:
\begin{equation}
    \label{eq:ML_transition}
    \sum_{t=1}^{k}\sum_{h=1}^H \left(\psi\left(s_{h+1}^{t}\right)-\mathbb{E}_{s_{h}^t, a_{h}^t}^{\thetahp_{k}}\left[\psi\left(s^{\prime}\right)\right]\right)^{\top} A_{i} \varphi\left(s_{h}^t, a_{h}^t\right)=\eta \nabla_{i} \operatorname{pen}\left(\thetahp_{k}\right).
\end{equation}
Equation~\eqref{eq:ML_transition} can be solved in closed form for simple distributions, like Gaussian, but it can involve integral approximations for other distribution (\cf Appendix~\ref{app:MLe}). We estimate the parameter for reward, \ie $\theta_r$, similarly
\begin{align}
    \thetahr(k) \in \argmin _{\theta \in \bR^{d}} \sum_{t=1}^{k}\sum_{h=1}^H -\log \bP_{\theta}\left(r_{t} \mid s_{h}^t, a_{h}^t\right) &+\eta \operatorname{pen}(\theta),\\
    \implies \qquad \sum_{t=1}^{k}\sum_{h=1}^H \left(r_t -\mathbb{E}_{s_{h}^t, a_{h}^t}^{\thetahr_{k}}\left[r\right]\right)\: B^{\top} A_{i} \varphi\left(s_{h}^t, a_{h}^t\right) =\eta &\nabla_{i} \operatorname{pen}\left(\thetahr_{k}\right)\quad \forall i \in [d].\label{eq:ML_reward}
\end{align}
\noindent\textbf{Exploration.}\label{sec:exploration}
A significant challenge in RL is handling exploration in continuous spaces. The majority of the literature is split between intractable, upper confidence bound-style optimism or Thompson sampling algorithms with high-dimensional posterior and guarantees only in terms of Bayesian regret. In \algo, we adopt the approach of reward perturbation motivated by the \RLSVI-framework~\cite{zanette2020frequentist,osband2016generalization}. We show that perturbing the reward estimation can guarantee optimism with a constant probability, \ie there exists $\nu \in (0,1]$ such that for all $k\in [K]$ and $s_1^k \in \cS$,
\begin{equation*}
    \prob{\tilde{V}_1(s_1^k) - V_1^\star(s_1^k)\ge0}\ge \nu.
\end{equation*} 
\cite{zanette2020frequentist} proves that this suffices to bound the learning error. However, their method clashes with not clipping the value function, as it modifies the probability of optimism. Thus, \cite{zanette2020frequentist} proposes an involved clipping procedure to handle the issue of unstable values. Instead, by careful geometric analysis (\cf Lemma~\ref{lem:worst_case_elliptical}), we bound the occurrences of the unstable values, and in turn, upper bound the regret without clipping. Note that unlike~\citep{ishfaq2021randomized}, \algo does not guarantee that the estimated value function is optimistic but still is able to control the learning error (\cf Section~\ref{sec:proof_outline}).

\noindent\textbf{Planning.}\label{sec:planning}
Recall that with our model assumptions, we can write the state-action value function linearly. Using \algo, we have at step $h$:

\begin{equation}\label{eq:linear_value_function}
    \tQ{h}^\pi (s, a) = \bE_{\thetatr}[r(s,a)] + \bigg\langle \pphi(s,a), \int_{\cS} \pmu(\tilde{s}) \tV{h+1}^\pi (\tilde{s}) d\tilde{s} \bigg\rangle.
\end{equation}

Then, we select the best action greedily using dynamic programming to compute $Q_h(s,a)$. Although our model yields infinite dimensional $\pphi$ and $\ppsi$, approximating them (\cf next paragraph) with linear features of dimension $\bigO(p H^2 K\log(H K))$ is possible without increasing the regret. Thus, the planning is done in $\bigO(p H^3 K\log(H K))$, which is pseudo-polynomial in $p$, $H$ and $K$, \ie tractable.

For details about the finite-dimensional approximation of our transition kernel, refer to Appendix~\ref{app:RFT}. Now, we highlight the schematic of a finite-dimensional approximation of $\pphi$ and $\ppsi$. We proceed in three steps. \textbf{1)} We have with high probability $\bS(V_{\thetahp,\thetatr ,h})\le d H^{3/2}$ (Section~\ref{sec:proof_outline}). \textbf{2)} If we have a uniform $\epsilon$-approximation of $\bP_{\thetap}$, we show that using it incurs at most an extra $\bigO(\epsilon d H^{5/2} K)$ regret. \textbf{3)} Finally, following~\cite{rahimi2007random}, we approximate uniformly the shift invariant kernels, here the RBF in Equation~\eqref{eq:linear_transition}, within $\epsilon$ error and with features of dimensions $\bigO(p \epsilon^{-2}\log\frac{1}{\epsilon^2})$, where $p$ is dimension of $\psi$. Associating these three elements and choosing $\epsilon = 1/\sqrt{(H^2 K)}$, we establish our claim.

\begin{remark}
The observation of linearity (\cf Equation.~\eqref{eq:linear_value_function} and Line~\ref{alg_2:planning}) does not reduce
BEF MDPs to linear MDPs because the former holds in an RKHS. Also, linearity is not in the representation parameter. Therefore, linear RL algorithms do not readily solve the BEF MDPs.
\end{remark}


\section{Theoretical analysis: proof outline}\label{sec:proof_outline}

To convey the novelties in our analysis, we provide a proof sketch for Theorem~\ref{thm:regret_bound}. We start by decomposing the regret into an estimation loss and a learning error, as given below
\begin{equation}
    \label{eq:regret_decomposition_estim_learn}
    R(K) = \sum_{k=1}^K (\V{1}^\star - \V{1}^{\pi_k})(s_{1 k}) = \sum_{k=1}^K (\underbrace{\V{1}^\star - \tV{1}^{\pi_k}}_{learning} + \underbrace{\tV{1}^{\pi_k} - \V{1}^{\pi_k}}_{Estimation})(s_{1 k}).
\end{equation}
For the \textbf{estimation error}, we use smoothness arguments with concentrations of parameters up to some novelties. Regarding the \textbf{learning error}, we show that the injected noise ensures a constant probability of anti-concentration. Applying Assumption~\ref{ass:bounds_hessians} and Lemma~\ref{lem:elliptical} leads to the upper-bound.

\subsection{Bounding the estimation error}
\label{sec:estimation_sketch}

We further decompose the estimation error into the errors in estimating transitions and rewards.
\begin{equation}
    \label{eq:estimation_error_decomposition}
    V_{\thetahp, \thetatr}^\pi(s_{1 k}) -V_{\thetap, \thetar}^\pi(s_{1 k}) = \underbrace{V_{\thetahp, \thetar}^\pi(s_{1 k}) -V_{\thetap, \thetar}^\pi(s_{1 k})}_{\text{transition estimation}} + \underbrace{V_{\thetahp, \thetatr}^\pi(s_{1 k}) - V_{\thetahp, \thetar}^\pi(s_{1 k})}_{\text{reward estimation}}
\end{equation}

\textbf{Transition estimation} Since the reward parameter is exact, the value function's span is $\le H$. Then, using the transportation of Lemma~\ref{lem:transportation} we obtain the bound {\small$H \sum_{h=1}^H \sqrt{2\operatorname{KL}_{s_{h k}, a_{h k}}(\thetap, \thetahp)}$}. We notice that since the reward parameter is exact, the bound is actually {\small$H \min\{1,\sum_{h=1}^H \sqrt{2\operatorname{KL}_{s_{h k}, a_{h k}}(\thetap, \thetahp)}\}$}. Using Lemma~\ref{lem:elliptical} under Assumption~\ref{ass:bounds_hessians}, we win a $\sqrt{H}$ factor compared to the analysis of~\cite{chowdhury2019online}.

\textbf{Reward estimation} Previous work uses clipping to help control this error, but in this case it can hinder the optimism probability by biasing the noise. \cite{zanette2020frequentist} proposes an involved clipping depending on the norms $\|(A_i \varphi(s_h^k,a_h^k))_{i\in[d]}\|_{(\pG)^{-1}}$, which is somewhat delicate to analyze and deploy. We remedy the situation acting solely in the proof. First let's define what we call the set of ``bad rounds'': $\left\{k \in [K], \exists h:\: \|(A_i \varphi(s_h^k,a_h^k))_{i\in[d]}\|_{(\pG)^{-1}}\ge 1\right\}$, these rounds are why clipping is necessary. Thanks to Lemma~\ref{lem:worst_case_elliptical}, we know that the number of such rounds is at most $\bigO(d)$. Surprisingly, it depends neither on $H$ nor on $K$. We show that the ``bad rounds'' incur at most $O(d^{3/2} H^2)$ regret, independent of $K$. Therefore, our algorithm can forgo clipping for free.

\begin{remark}
If it wasn't for the episodic nature of our setting, we could have used the forward algorithm to eliminate the span control issue. We refer to \cite{vovk2001competitive,azoury2001relative} for a description of this algorithm, \cite{ouhamma2021stochastic} for a stochastic analysis, and Section 4 therein for an application to linear bandits.
\end{remark}

\subsection{Bounding the learning error}
To upper-bound this term of the regret, we first show that the estimated value function is optimistic with a constant probability. Then, we show that this is enough to control the learning error.

\textbf{Stochastic optimism.} The perturbation ensures a constant probability of optimism. Specifically,
\begin{align*}
    (\tV{1}-&\V{1}^\star)(s_1) \ge (\tQ{1}^\star-Q_{1}^\star)(s_1,\pi^\star (s_1))\\
    &\ge \underbrace{V_{\thetahp,\thetar}^{\pi^\star} (s_1) - V_{\thetap,\thetar}^{\pi^\star} (s_1)}_{\text{first term}} + \underbrace{V_{\thetahp,\thetahr}^{\pi^\star} (s_1)  - V_{\thetahp,\thetar}^{\pi^\star} (s_1)}_{\text{second term}} + \underbrace{V_{\thetahp,\thetatr}^{\pi^\star} (s_1) - V_{\thetahp,\thetahr}^{\pi^\star} (s_1)}_{\text{third term}}
\end{align*}
The first and second terms are perturbation free, we handle them similarly to the estimation error, \ie using concentration arguments for $\thetahp$ and $\thetahr$. For the third term, we use transportation of rewards (Lemma~\ref{lem:Reward_Transportation}) and anti-concentration of $\xi_k$ (Lemma~\ref{lem:gaussian_anti_concentration}). We find that with probability at least $1-2 \delta$
\begin{align*}
    (\tV{1}-\V{1}^\star)(s_1) \ge& \xi_k ^\top \: \bE_{(\tilde{s}_t)_{t \in [H]}\sim \thetahp \mid s_1^k}\left[\sum_{t=1}^H \frac{\Var^{\thetar_j} (r)}{2} (A_i \varphi(\tilde{s}_t,\pi^\star (\tilde{s}_t)))_{i\in[d]}\right] B\\
    &{\displaystyle - H c(n,\delta) \left\|\sum_{h=1}^H \bE_{(\tilde{s}_t)_{t \in [H]}\sim \thetahp \mid s_1^k} \left[ (A_i \varphi(\tilde{s}_h, \pi^\star (\tilde{s}_h)))_{i\in [d]}\right] \right\|_{(\pG)^{-1}}},
\end{align*}
where {\small $c(n,\delta)\!= \!\left(\!\sqrt{\!\betap \betap(n,\delta)/\alphap}\! +\! \sqrt{\!\betar \betar(n,\delta) \min\{1,\alphap/\alphar\}/(2 \alphar)\!} \!\right)$}. Since $\xi_k \!\sim\! \cN(0,x_k (\pG)^{-1}\!)$ and $x_k \!\ge\! H^2 c(n,\delta)^2$, we get $\prob{\tV{1}^\pi(s_1)-\V{1}^\star(s_1) \ge 0}\ge \Phi(-1)$, where $\Phi$ is the normal CDF. This is ensured by the anti-concentration property of Gaussian random variables, see Lemma~\ref{lem:gaussian_anti_concentration}.

\textbf{From stochastic optimism to error control:} Existing algorithms require the value function to be optimistic (\ie negative learning error) with large probability. Contrary to them, \algo only requires the estimated value to be optimistic with a constant probability. When it is, the learning happens. Otherwise, the policy is still close to a good one thanks to the decreasing estimation error, and the learning still happens. This part of the proof is similar in spirit to that of~\cite{zanette2020frequentist}.

\underline{\textit{Upper bound on $V_1^\star$:}} Draw $(\bar{\xi}_{k})_{k\in [K]}$ i.i.d copies of $(\xi_{k})_{k\in [K]}$ and define the event where optimism holds as $\bar{O}_k \triangleq \{V_{\thetahp,\thetatr_k,1}(s_1^k) - V_1^\star (s_1^k) \ge 0\}$. This implies that $\quad V_1^\star (s_1^k) \le \bE_{\bar{\xi}_k \mid \bar{O}_k }[V_{\thetahp,\thetahr+\bar{\xi}_k,1}(s_1^k)].$

\underline{\textit{Lower bound on $V_{\thetahp,\thetatr}\:$:}} Consider $\underbar{V}_{1}(s_1^k)$ to be a solution of the optimization problem
\begin{equation*}
    \min _{\xi_{ k}} V_{\thetahp, \thetahr+\xi_k,1}(s_{1}^k) \quad \text{ subject to: } \: \|\xi_k\|_{\bar{G}_k} \le \sqrt{x_k d\log(d/\delta)},
\end{equation*}
As the injected noise concentrates, we obtain $\underbar{V}_{1}(s_1^k) \le V_{\thetahp,\thetatr}(s_1^k)$.

\underline{\textit{Combination:}} Using these upper and lower bounds, we show that with probability at least $1-\delta$,
\begin{align*}
    V_1^\star (s_1^k) -& V_{\thetahp,\thetahr+\bar{\xi}_k,1}(s_1^k) \le \bE_{\bar{\xi}_k \mid \bar{O}_k }[V_{\thetahp,\thetahr+\bar{\xi}_k,1}(s_1^k)-\underbar{V}_{1}(s_1^k)]\\
    &\le \left(\bE_{\bar{\xi}_k }[V_{\thetahp,\thetahr+\bar{\xi}_k,1}(s_1^k)-\underbar{V}_{1}(s_1^k)] - \bE_{\bar{\xi}_k \mid \bar{O}_k^\texttt{c} }[V_{\thetahp,\thetahr+\bar{\xi}_k,1}(s_1^k)-\underbar{V}_{1}(s_1^k)]\bP(\bar{O}_k^\texttt{c})\right)/ \bP(\bar{O}_k),
\end{align*}
The last step follows from the tower rule. Note that the term inside the expectations is positive with high probability but not necessarily in expectation. We follow the lines of the estimation error analysis to complete the proof of Theorem~\ref{thm:regret_bound}. We refer to Appendix~\ref{app:learning_error} for the detailed proof.

\section{Related work: functional representation in RL with regret and tractability}\label{sec:related_work}

Our work extends the endeavor of using functional representations for regret minimization in continuous state-action
MDPs. Now, we posit our contributions in existing literature.

\textit{Kernel value function representation.} \cite{pmlr-v119-ayoub20a} studies MDPs with a linear mixtures model then extends to an RKHS setting, this generalizes our work and that of \cite{yang2020reinforcement}. However, the paper proposes an Eluder-dimension analysis, for RKHS settings this leads to the result of \cite{yang2020reinforcement}, \ie a regret $H\log(T)^d$ higher than for \algo.Recently,~\cite{huang2021short} shows that for RKHS, Eluder dimension and the information gain are strictly equivalent, which brings in the extra factor.

\textit{General functional representation.} The Eluder dimension is a complexity measure often used to analyze RL with general function space, \cite{huang2021short} asserts that "common examples of where it is known to be small are function spaces (vector spaces)". \citep{dai2018sbeed} provides the first convergence guarantee for general nonlinear function representations in the Maximum Entropy RL setting, where entropy of a policy is used as a regularizer to induce exploration. Thus, the analysis cannot address episodic RL, where we have to explicitly ensure exploration with optimism. In the episodic setting, \citep{wang2020reinforcement} leverage the UCB approach for tabular MDPs and function spaces with bounded Eluder dimension, this strategy achieves a and achieve a $\tilde{O}\left(\sqrt{d^4 H^2 T}\right)$ regret for linear MDPs. \cite{ishfaq2021randomized} considers the same setting, proposes an \RLSVI based algorithm, and achieves a $\tilde{O}(\sqrt{d^3 H^4 K})$ for linear MDPs. However, the latter assumes an oracle perturbing the estimation to achieve anti-concentration while maintaining a bounded covering number, which is a counter-intuitive mix of boundedness and anti-concentration. Indeed,  \cite{zanette2020frequentist} studied the linear MDP case, and while it managed to design an ingenious clipping verifying previous assumptions, the method is extremely intricate and the proof is involved and unlikely to extend for general value function spaces. \textit{To concertize our design, we focus on the general but explicit BEF of MDPs than any abstract representation. We also remove the requirement to clip with a novel analysis.}

\textit{Bilinear exponential family of MDPs.} Exponential families are studied widely in RL theory, from bandits to MDPs \cite{lu2021low, korda2013thompson,filippi2010parametric,kveton2006solving}, as an expressive parametric family to design theoretically-grounded model-based algorithms. \cite{chowdhury2021reinforcement} first studies episodic RL with Bilinear Exponential Family (BEF) of transitions, which is linear in both state-action pairs and the next-state. It proposes a regularized log-likelihood method to estimate the model parameters, and two optimistic algorithms with upper confidence bounds and posterior sampling. Due to its generality to unifiedly model tabular MDPs, factored MDPs, and linearly controlled dynamical systems, the BEF-family of MDPs has received increasing attention~\citep{li2021exponential}. \cite{li2021exponential} estimates the model parameters based on score matching that enables them to replace regularity assumption on the log-partition function with Fisher-information and assumption on the parameters. Both \citep{chowdhury2021reinforcement,li2021exponential} achieve a worst-case regret of order $\Tilde{O}(\sqrt{d^2 H^4 K})$ for known reward. On a different note, \citep{du2021bilinear,foster2021statistical} also introduces a new structural framework for generalization in RL, called bilinear classes as it requires the Bellman error to be upper bounded by a bilinear form. Instead of using bilinear forms to capture non-linear structures, this class is not identical to BEF class of MDPs, and studying the connection is out of the scope of this paper.
Specifically, \textit{we address the shortcomings of the existing works on BEF-family of MDPs that assume known rewards, absence of RLSVI-type algorithms, and access to oracle planners.} 

\textit{Tractable planning and linearity.} Planning is a major byproduct of the chosen functional representation. In general, planning can incur high computational complexity if done na\"ively. Specially, \cite{du2019good} shows that for some settings, even with a linear $\epsilon$-approximation of the $Q$-function, a planning procedure able to produce an $\epsilon$-optimal policy has a complexity at least $2^H$. Thus, different works~\citep{shariff2020efficient, lattimore2020learning, van2019comments} propose to leverage different low-dimensional representations of value functions or transitions to perform efficient planning. Here, we take note from~\citep{ren2021free} that Gaussian transitions induce an explicit linear value function in an RKHS. And generalize this observation with the bilinear exponential. Moreover, using uniformly good features~\cite{rahimi2007random} to approximate transition dynamics from our model enables us to design a tractable planner. We provide a detailed discussion of this approximation in Section~\ref{sec:building_blocks}. More practically, \cite{ren2021free,nachum2021provable} use representations given by random Fourier features~\citep{rahimi2007random} to approximate the transition dynamics and provide experiments validating the benefits of this approach for high-dimensional Atari-games. 
\textit{Thus, we propose the first algorithm with tractable planning for BEF-family.}

\vspace*{-.4em}
\section{Conclusion and future work}
We propose the \algo algorithm for the bilinear exponential family of MDPs in the setting of episodic-RL. \algo explores using a Gaussian perturbation of rewards, and plans tractably (complexity of $\bigO(p H^3 K\log(H K))$) thanks to properties of the RBF kernel. Our proof shows that clipping can be forwent for similar \RLSVI-type algorithms. Moreover, we prove a $\sqrt{d^3 H^3 K}$ frequentist regret bound, which improves over existing work, accommodates unknown rewards, and matches the lower bound in terms of $H$ and $K$. 
Regarding future work, we believe that our proof approach can be extended to rewards with bounded variance. We also believe that the extra $\sqrt{d}$ in our bound is an artefact of the proof, and specifically, the anti-concentration. We will investigate it further. Finally, we plan to study the practical efficiency of \algo through experiments on tasks with continuous state-action spaces in an extended version of this work.

\section*{Acknowledgments and Disclosure of Funding}

The authors acknowledge the funding of the French National Research Agency, the French Ministry of Higher Education and Research, Inria, the MEL and the I-Site ULNE regarding project R-PILOTE-19-004-APPRENF. R. Ouhamma also acknowledges support from Ecole polytechnique.

\bibliography{references}
\newpage

\newpage
\appendix

\part{Appendix}
\parttoc
\newpage

\section{Notations}\label{app:notations}
We dedicate this section to index all the notations used in this paper. Note that every notation is defined when it is introduced as well.

\renewcommand{\arraystretch}{1.5}
\begin{longtable}{l l l}
\caption{Notations}\label{tab:Notation}\\
\hline
 $H$ &$\eqdef$ & number of steps in a given episode\\
 $K$ &$\eqdef$ & number of episodes\\
 $T$ &$\eqdef$ & $K H$, total number of steps\\
 $s_{h}^k$ &$\eqdef$ & state at time $h$ of episode $k$,  denoted $s_h$ when $k$ is clear from context \\
 $a_{h}^k$ &$\eqdef$ & action at time $h$ of episode $k$, denoted $a_h$ when $k$ is clear from context\\
 $r(s,a)$ &$\eqdef$ & realization of the reward in state $s$ under action $a$\\
 $\thetap$ &$\eqdef$ & parameter of the transition distribution, $\in \bR^d$\\
 $\thetar$ &$\eqdef$ & parameter of the reward distribution, $\in \bR^d$\\
 $\theta$ &$\eqdef$ & $\in \bR^d$ denotes either $\thetar$ or $\thetap$, unless stated otherwise\\
 $\hat{\theta}$ &$\eqdef$ & $\theta$ estimator with Maximum Likelihood unless stated otherwise\\
 $\tilde{\theta}$ &$\eqdef$ & $\hat{\theta}+\xi$ where $\xi$ is a chosen noise. Perturbed estimation of $\theta$.\\
 $[\theta_1, \theta_2]$ &$\eqdef$ & the d-dimensional $\ell_\infty$ hypercube joining $\theta_1$ and $\theta_2$ \\
 $\bP_{\thetap}$ &$\eqdef$ & transition under the exponential family model with parameter $\thetap$\\
 $\psi$ &$\eqdef$ & feature function, $\in (\bR_+^p)^\cS$\\
 $\varphi$ &$\eqdef$ & feature function, $\in (\bR_+^q)^{\cS\times\cA}$\\
 $B$ &$\eqdef$ & $p$-dimensional vector\\
 $M_\theta$ &$\eqdef$ & $\sum_{i=1}^d \theta_i A_i$, where $A_i$ are $p\times q$ matrices.\\
 $\rz$ &$\eqdef$ & the rewards' log partition function\\
 $\pz$ &$\eqdef$ & the transitions' log partition function\\
 $\mathcal{H}$ &$\eqdef$ & Hilbert space where we decompose transitions\\
 $\pmu$ &$\eqdef$ & feature function after decomposition, $\in (\bR_+)^{\cS \times \mathcal{H}}$\\
 $\pphi$ &$\eqdef$ & feature function after decomposition, $\in (\bR_+)^{\cS \times \cA \times \mathcal{H}}$ \\
 $G_{s, a}$ &$\eqdef$ & $\left(\varphi(s, a)^{\top} A_{i}^{\top} A_{j} \varphi(s, a)\right)_{i, j \in [d]}$ \\
 $\bar{G}_k^\texttt{r}$ &$\eqdef$ & $ \bar{G}_{(k-1)h}^\texttt{r} = \frac{\eta}{\alphar}\bA + \sum_{\tau=1}^{k-1} \sum_{h=1}^H G_{s_h^\tau , a_h^\tau}$ \\
 $\bar{G}_k^\texttt{p}$ &$\eqdef$ & $\bar{G}_{(k-1)h}^\texttt{p} = \frac{\eta}{\alphap}\bA + \sum_{\tau=1}^{k-1} \sum_{h=1}^H G_{s_h^\tau , a_h^\tau}$ \\
 $\mathbb{C}_{s, a}^{\theta}\left[\psi\left(s^{\prime}\right)\right]$ &$\eqdef$ & $\mathbb{E}_{s, a}^{\theta}\left[\psi\left(s^{\prime}\right) \psi\left(s^{\prime}\right)^{\top}\right]-\mathbb{E}_{s, a}^{\theta}\left[\psi\left(s^{\prime}\right)\right] \mathbb{E}_{s, a}^{\theta}\left[\psi\left(s^{\prime}\right)^{\top}\right]$ \\
 $\betap$ &$\eqdef$ & $\sup _{\theta, s, a} \lambda_{\max }\left(\mathbb{C}_{s, a}^{\theta}\left[\psi\left(s^{\prime}\right)\right]\right)$ linked to the maximum eigenvalue of $\nabla^2 \pz$\\
 $\alphap$ &$\eqdef$ & $\inf _{\theta, s, a} \lambda_{\max }\left(\mathbb{C}_{s, a}^{\theta}\left[\psi\left(s^{\prime}\right)\right]\right)$ linked to the minimum eigenvalue of $\nabla^2 \pz$ \\
 $\betar$ &$\eqdef$ & $\lambda_{\max }\left(B B^\top \right) \sup _{\theta, s, a}  \Var_{s,a}^\theta(r)$, linked to the maximum eigenvalue of $\nabla^2 \rz$ \\
 $\alphar$ &$\eqdef$ & $\lambda_{\min }\left(B B^\top \right) \inf _{\theta, s, a}  \Var_{s,a}^\theta(r)$, linked to the minimum eigenvalue of $\nabla^2 \rz$ \\
 \bottomrule
\end{longtable}

\newpage
\section{Regret analysis}\label{app:regret_analysis}

We provide a high probability analysis of the regret of \algo under standard regularity assumptions of the representation. First we recall the regret definition then we separate the perturbation error from the statistical estimation:
\begin{align*}
    \cR (K) = \sum_{k=1}^K (\V{1}^\star - \V{1}^{\pi_k})(s_{1}^k) = \sum_{k=1}^K \Big(\underbrace{\V{1}^\star - \tV{1}^{\pi_k}}_{learning} + \underbrace{\tV{1}^{\pi_k} - \V{1}^{\pi_k}}_{Estimation}\Big)(s_{1}^k)
\end{align*}



\subsection{Estimation error}
To show that the estimation error $\big(\sum_{k=1}^K \tV{1} - \V{1}^{\pi_k}\big)$ can be controlled, we decompose it to an error that comes from the estimation of the transition parameter and one that comes from the estimation of the reward parameter:

\begin{equation*}
    V_{\thetahp, \thetatr}^\pi( s_{1}^k) -V_{\thetap, \thetar}^\pi(s_{1}^k) = \underbrace{V_{\thetahp, \thetar}^\pi(s_{1}^k) -V_{\thetap, \thetar}^\pi(s_{1}^k)}_{\text{transition estimation}} + \underbrace{V_{\thetahp, \thetatr}^\pi(s_{1}^k) - V_{\thetahp, \thetar}^\pi( s_{1}^k)}_{\text{reward estimation}},
\end{equation*}
we control each term separately in Section~\ref{app:transition_estimation} and Section~\ref{app:reward_estimation}. Therefore, we obtain the following lemma controlling the estimation error.

\begin{lemma}
    \label{lem:estimation_error}
    The estimation error satisfies, with probability at least $1-5\delta$
    \begin{align*}
        \sum_{k=1}^K &\tV{1}(s_{1}^{k}) -\V{1}^\pi(s_{1}^{k}) \le 2 H \sqrt{\frac{2\betap}{\alphap} \betap(N,\delta) N \gamma_{K}^\texttt{p}} + 2H\sqrt{2 N \log(1/\delta)}\\
        +& \left[ \sqrt{K H d \log \left(1+\alphar \eta^{-1} B_{\varphi, \mathbb{A}} n\right)} + C_d \sqrt{H d\log(1+\alpha \eta^{-1} B_{\varphi, \bA}H)} \right] \times \Bigg( \sqrt{ \frac{\betar(n,\delta) }{2\alphar}}\\
        +& c\sqrt{(\max_{k} x_k)d\log(d K /\delta)}\Bigg) \betar + \sqrt{2K H d \log \left(1+\alphar \eta^{-1} B_{\varphi, \mathbb{A}} n\right) \log(1/\delta)}
    \end{align*}
    where for $\texttt{i}\in [\texttt{p}, \texttt{r}]$, $\beta^\texttt{i}(K,\delta)\triangleq \frac{\eta}{2}B_{\bA}^2+\gamma_K^\texttt{i}+\log(1/\delta)$, and $\gamma_K^\texttt{i}\triangleq d\log(1+\frac{\beta^\texttt{i}}{\eta}B_{\varphi,\bA}H K)$. Also, $C_d \triangleq\frac{3d}{\log(2)}\log\left(1+ \frac{\alphar\|\bA\|_2^2 B_{\varphi,\bA}^2}{\eta \log(2)}\right) $, and $c$ is a universal constant.
\end{lemma}

\begin{proof}
It follows directly by combining Lemma~\ref{lem:transition_estimation_bound} and Lemma~\ref{lem:reward_estimation_bound} using a union bound.
\end{proof}

\subsubsection{Transition estimation}\label{app:transition_estimation}
The goal of this section is to prove the following lemma which bounds the regret due to transition estimation.

\begin{lemma}
\label{lem:transition_estimation_bound}
We have, with probability at least $1-2\delta$
\begin{equation*}
    \sum_{k=1}^K V_{\thetahp, \thetar}(s_{1}^{k}) -V_{\thetap, \thetar}^\pi(s_{1}^{k}) \le 
    2 H \sqrt{\frac{2\betap}{\alphap} \betap(N,\delta) N \gamma_{K}^\texttt{p}} + 2H\sqrt{2 N \log(1/\delta)}
\end{equation*}
where $\gamma_{K}^\texttt{p}:=d \log \left(1+\betap \eta^{-1} B_{\varphi, \mathbb{A}} H K\right)$, and $\betap(K,\delta)\triangleq \frac{\eta}{2}B_{\bA}^2+\gamma_K^\texttt{p}+\log(1/\delta)$.
\end{lemma}

\begin{proof}
The proof proceeds in two parts. First, we will reveal a bound in terms of the induced local geometry, \ie a bound in terms of KL-divergence. Second, we explicit the bound by transferring the induced local geometry to the euclidean one.

\textit{\textbf{1)} Bound in terms of local geometry.} \quad We provide a bound on the estimation error of the transition in terms of $\operatorname{KL}$ divergences, for that end we show that the estimation error can be decomposed and well controlled. We start by writing the one-step decomposition:

\begin{align*}
    V_{\thetahp, \thetar,1}^\pi(s_{1}^{k}) -&V_{\thetap, \thetar,1}^\pi (s_{1}^{k})\\
    &=\bE_{s_{1}^{k},a_{1}^{k}}^{\thetahp}\left[V_{\thetahp, \thetar,2}^\pi\right] -  \bE_{s_{1}^{k},a_{1}^{k}}^{\thetap}\left[V_{\thetahp, \thetar,2}^\pi\right] + \bE_{s_{1}^{k},a_{1}^{k}}^{\thetap}[V_{\thetahp, \thetar,2}^\pi - V_{\thetap, \thetar,2}^\pi]\\
    &= \bE_{s_{1}^{k},a_{1}^{k}}^{\thetahp}\left[V_{\thetahp, \thetar,2}^\pi\right] -  \bE_{s_{1}^{k},a_{1}^{k}}^{\thetap}\left[V_{\thetahp, \thetar,2}^\pi\right] + V_{\thetahp, \thetar,2}^\pi(s_{2 k}) -V_{\thetap, \thetar,2}^\pi(s_{2 k}) + \zeta_{1}^{k}\\
    &= \sum_{h=1}^H \bE_{s_{h k},a_{h k}}^{\thetahp}\left[V_{\thetahp, \thetar,h+1}^\pi\right] -  \bE_{s_{h k},a_{h k}}^{\thetap}\left[V_{\thetahp, \thetar,h+1}^\pi\right] + \zeta_{h k}
\end{align*}
where $\zeta_{h k} = \bE_{s_{h k},a_{h k}}^{\thetap}[V_{\thetahp, \thetar,h+1}^\pi - V_{\thetap, \thetar,h+1}^\pi] - \left(V_{\thetahp, \thetar,h+1}^\pi(s_{h+1 k}) -V_{\thetap, \thetar,h+1}^\pi(s_{h+1 k})\right)$ is a martingale sequence, and the last equality comes by induction. Here we consider the true reward parameter which verifies $|\bE_{\thetar}[r(s,a)]|\le 1$ by assumption, therefore $|\zeta_{h k}|\le 2 H$.  Using the Azuma-Hoeffding inequality \cite{boucheron2013concentration}, with probability at least $1-\delta$
\begin{equation*}
    \sum_{k=1}^K \sum_{h=1}^H \zeta_{h k} \le 2H \sqrt{2K H \log(1/\delta)}
\end{equation*}
We finish bounding the first term using Lemma~\ref{lem:transportation}, indeed 
\begin{align*}
    \bE_{s_{h k},a_{h k}}^{\thetahp}\left[V_{\thetahp, \thetar,h+1}^\pi\right] -  \bE_{s_{h k},a_{h k}}^{\thetap}\left[V_{\thetahp, \thetar,h+1}^\pi\right] &\le H \sqrt{2 \operatorname{KL}_{s_{h k}, a_{h k}}(\thetap, \thetahp)}\\
    &\le H \min\left\{1, \sqrt{2 \operatorname{KL}_{s_{h k}, a_{h k}}(\thetap, \thetahp)}\right\},
\end{align*}
the last inequality follows because $\forall h, \: \bS(V_{\thetahp,\thetar,h+1})\le H$. 
\begin{remark}
Traditionally, the expected value difference bound follows from the simulation lemma~\cite{ren2021free}. The simulation lemma incurs an extra $\sqrt{H}$ factor compared to our bound.
\end{remark}
We deduce that with probability at least $1-\delta$:
\begin{align}
    \sum_{k=1}^K V_{\thetahp, \thetar}(s_{1}^{k})& -V_{\thetap, \thetar}^\pi(s_{1}^{k}) \nonumber\\
    &\quad \le H \sum_{k=1}^K \min\left\{1,\sum_{h=1}^H \sqrt{2\operatorname{KL}_{s_{h k}, a_{h k}}(\thetap, \thetahp)}\right\}+2H\sqrt{2 K H \log(1/\delta)}  \label{eq:transition_estimation_kl}
\end{align}

\textit{\textbf{2)} Bounding the sum of KL divergences.} \quad we explicit the bound of inequality~\eqref{eq:transition_estimation_kl} using Assumption~\ref{ass:bounds_hessians} along with properties of the exponential family (\cf Section~\ref{app:properties_exp_fam}). We have for all $(s, a)$,
\begin{equation}
    \label{ineq:transition_kl}
    \forall \thetap, \thetapp, \quad \frac{\alphap}{2}\left\|\thetapp-\thetap\right\|_{G_{s, a}}^{2} \leq \mathrm{KL}_{s, a}\left(\thetap, \thetapp\right) \leq \frac{\betap}{2}\left\|\thetapp-\thetap\right\|_{G_{s, a}}^{2}.
\end{equation}
This implies that
\begin{equation*}
    \mathrm{KL}_{s, a}\left(\thetahp(k), \thetap\right) \leq \frac{\betap}{2}\left\|\thetap-\thetahp(k)\right\|_{G_{s, a}}^{2} \leq \betap \left\|(\bar{G}_{k}^\texttt{p})^{-1 / 2} G_{s, a} (\bar{G}_{k}^\texttt{p})^{-1 / 2}\right\| \frac{1}{2}\left\|\thetap-\thetahp(k)\right\|_{\bar{G}_{k}^\texttt{p}}^2,
\end{equation*}
where $\bar{G}_{k}^\texttt{p} \equiv \bar{G}_{(k-1) H}^\texttt{p}:=G_{k}+(\alphap)^{-1} \eta \bA$ and $G_{k} \equiv \sum_{\tau=1}^{k-1} \sum_{h=1}^{H} G_{s_{s}^{\tau}, a_{h}^{\tau}}$.

From Corollary~\ref{cor:Transition_euclidean_confidence_region}, with probability at least $1-\delta$ and for all $k\in \bN$
\[\left\|\thetap-\thetahp(k)\right\|_{\bar{G}_{k}^\texttt{p}}^2 \le 2\betap(k,\delta) / \alphap.\]

Also, using Lemma~\ref{lem:elliptical}, we have
\begin{equation*}
    \sum_{t=1}^{T} \sum_{h=1}^{H}\min\left\{1, \left\|(\bar{G}_{k}^\texttt{p})^{-1 / 2} G_{s, a} (\bar{G}_{k}^\texttt{p})^{-1 / 2}\right\|\right\} \le 2 d \log \left(1+\alphap \eta^{-1} B_{\varphi, \mathbb{A}} H K\right).
\end{equation*}
Combining these two results we obtain, with probability at least $1-\delta$:

\begin{equation}
    \label{eq:sum_kl_bound}
    \sum_{t=1}^{T} \sum_{h=1}^{H} \min\left\{1, \mathrm{KL}_{s_{h}^{t}, a_{h}^{t}}\left(\thetahp(k), \thetap\right)\right\} \leq \frac{2 \betap}{\alphap} \betap(K,\delta) \gamma_{K}^\texttt{p}.
\end{equation}

\begin{remark}
Notice that the minimum with $1$ is crucial, indeed, without it the bound deteriorates by a factor $H$ as was the case in~\cite{chowdhury2021reinforcement}.
\end{remark}

\textit{\textbf{3)} Combining the bounds.} \quad By applying Cauchy-Schwarz in inequality~\eqref{eq:transition_estimation_kl}, we obtain, with probability at least $1-\delta$, and for all $K \in \bN$
\begin{equation*}
    \sum_{k=1}^K V_{\thetahp, \thetar}(s_{1}^{k}) -V_{\thetap, \thetar}^\pi(s_{1}^{k}) \le H \sqrt{2\sum_{k=1}^K \sum_{h=1}^H \operatorname{KL}_{s_{h k}, a_{h k}}(\thetap, \thetahp)}+2H\sqrt{2 K H \log(1/\delta)}.
\end{equation*}
Injecting inequality~\eqref{eq:sum_kl_bound} proves the desired result with probability at least $1-2\delta$.
\end{proof}

\subsubsection{Reward estimation}\label{app:reward_estimation}

Now, we provide the bound over the regret due to estimating the reward parameter.

\begin{lemma}
\label{lem:reward_estimation_bound}
With probability at least $1-3\delta$, the following result holds true.
\begin{align*}
    \sum_{k=1}^K  V_{\thetahp, \thetatr,1}^\pi (s_{1}^{k}) - &V_{\thetahp, \thetar,1}^\pi (s_{1}^{k}) \le \left(\sqrt{\frac{\betar(K,\delta)}{2\alphar}}+c\sqrt{(\max_{k\le K} x_k) d\log(d K /\delta)}\right) \betar \\
    \times & \left( \sqrt{C_d \left(1+\frac{\alphar B_{\varphi, A} H}{\eta}\right)} + \sqrt{K\log(e/\delta^2)} \right) \sqrt{ H d \log \left(1+\alphar \eta^{-1} B_{\varphi, \mathbb{A}} H K\right)},
\end{align*}
where $\beta^\texttt{p}(K,\delta)\triangleq \frac{\eta}{2}B_{\bA}^2+\gamma_K^\texttt{p}+\log(1/\delta)$, and $\gamma_K^\texttt{p}\triangleq d\log(1+\frac{\beta^\texttt{p}}{\eta}B_{\varphi,\bA}H K)$. Also, $C_d \triangleq\frac{3d}{\log(2)}\log\left(1+ \frac{\alphar\|\bA\|_2^2 B_{\varphi,\bA}^2}{\eta \log(2)}\right) $, and $c$ is a universal constant.
\end{lemma}

\begin{proof}
The reward estimation error in Equation~\eqref{eq:estimation_error_decomposition} can be written explicitly. Indeed, using Lemma~\ref{lem:Reward_Transportation}

\begin{align*}
    \tV{1}^\pi(s_{1}^{k}) - V_{\thetahp, \thetar,1}^\pi &(s_{1}^{k}) = \bE_{(\tilde{s}_h)_{1\le h\le H}\sim \pi \mid \thetahp, s_{1}^{k}}\left[\sum_{h=1}^H \frac{\Var_{\tilde{s}_h,\pi(\tilde{s}_h)}(r)}{2} B^\top M_{\thetatr-\thetar}\varphi(\tilde{s}_h,\pi(\tilde{s}_h))\right] \\
    \le& \bE\left[\sum_{h=1}^H \frac{\Var_{\tilde{s}_h,\pi(\tilde{s}_h)}(r)}{2} \|\thetatr-\thetar\|_{\rG} \|(B^\top A_i\varphi(\tilde{s}_h,\pi(\tilde{s}_h)))_{1\le i\le d}\|_{(\rG)^{-1}}\right]\\
    \le& \|\thetatr-\thetar\|_{\rG}  \bE\left[\sum_{h=1}^H \frac{\Var_{\tilde{s}_h,\pi(\tilde{s}_h)}(r)}{2} \|(B^\top A_i\varphi(\tilde{s}_h,\pi(\tilde{s}_h)))_{1\le i\le d}\|_{(\rG)^{-1}}\right]\\
    \le& \|\thetatr-\thetar\|_{\rG} \frac{\betar}{2}  \bE \Bigg[ \underbrace{\sum_{h=1}^H \|(A_i\varphi(\tilde{s}_h,\pi(\tilde{s}_h)))_{1\le i\le d}\|_{(\rG)^{-1}}}_{\eqdef \widetilde{\operatorname{traj}}_{k}} \Bigg],
\end{align*}
where $\operatorname{traj}_{k} \eqdef \sum_{h=1}^H \|(A_i\varphi(s_h,\pi(s_h)))_{1\le i \le d} \|_{ (G_{k }^\texttt{r} )^{-1}}$.

\textbf{Bad rounds.} We separate the analysis of this estimation error into bad and good rounds. Here we analyze the bad rounds, which are define by the following set:
\begin{equation*}
    \mathcal{T} = \{k\in \mathbb{N}^*, \exists h \in [H], \|(A_i\varphi(\tilde{s}_h,\pi(\tilde{s}_h)))_{1\le i\le d}\|_{(\rG)^{-1}} \ge 1\}
\end{equation*}

\textit{1)} We know that $\|(A_i\varphi(\tilde{s}_h,\pi(\tilde{s}_h)))_{1\le i\le d} (A_i\varphi(\tilde{s}_h,\pi(\tilde{s}_h)))_{1\le i\le d}^\top\|_2^2 \le \|\bA\|_2^2 B_{\varphi,\bA}^2$. Consequently, according to Lemma~\ref{lem:worst_case_elliptical}
\begin{equation*}
    \left|\mathcal{T}\right| \le \frac{3d}{\log(2)}\log\left(1+ \frac{\alpha\|\bA\|_2^2 B_{\varphi,\bA}^2}{\eta \log(2)}\right).
\end{equation*}

\textit{2)} Since $G_{k}$ is positive semi-definite, we have $\rG \succeq (\alphar)^{-1} \eta \bA$, and in turn, for all state-action couples $(s,a)$,
$\left\|(\rG)^{-1} G_{s, a}\right\| \leq \frac{\alphar}{\eta}\left\|\bA^{-1} G_{s, a}\right\| \leq \frac{\alphar B_{\varphi, \bA}}{\eta}$.

This further yields 
\[\left\|I+(\rG)^{-1} \sum_{h=1}^{H} G_{s_{h}^{t}, a_{h}^{t}}\right\| \leq 1+\sum_{h=1}^{H}\left\|(\rG)^{-1} G_{s_{h}^{t}, a_{h}^{t}}\right\| \leq 1+\frac{\alphar B_{\varphi, \bA} H}{\eta} .\]

Let us define $\bar{G}_{k+H}^\texttt{r}:=\rG+\sum_{h=1}^{H} G_{s_{h}^{k}, a_{h}^{k}}$. Then,
\begin{equation*}
    \bar{G}_{k+H}^{-1} G_{s, a}=\left(I+(\rG)^{-1} \sum_{h=1}^{H} G_{s_{h}^{t}, a_{h}^{t}}\right)^{-1} (\rG)^{-1} G_{s, a}.
\end{equation*}
Therefore, for all pairs $(s,a)$,
\begin{align*}
    \|(A_i\varphi(\tilde{s}_h,\pi(\tilde{s}_h)))_{1\le i\le d}\|_{(\rG)^{-1}} & = \sqrt{ \tr((A_i\varphi( \tilde{s}_h,\pi(\tilde{s}_h)))_{1\le i\le d}^\top (\rG)^{-1} (A_i\varphi(\tilde{s}_h,\pi(\tilde{s}_h)))_{1\le i\le d})}\\
    &= \sqrt{\tr(\left(1+\frac{\alphar B_{\varphi, A} H}{\eta}\right) (\bar{G}_{k+H}^\texttt{r})^{-1} G_{s, a} )} \\
    &\le \sqrt{\left(1+\frac{\alphar B_{\varphi, A} H}{\eta}\right)} \|(A_i\varphi(\tilde{s}_h,\pi(\tilde{s}_h)))_{1\le i\le d}\|_{(\bar{G}_{k+H}^\texttt{r})^{-1}}
\end{align*}
Since $\|(A_i\varphi(\tilde{s}_h,\pi(\tilde{s}_h)))_{1\le i\le d}\|_{(\bar{G}_{k+H}^\texttt{r})^{-1}} \le 1$, we have $\|(A_i\varphi(\tilde{s}_h,\pi(\tilde{s}_h)))_{1\le i\le d}\|_{(\bar{G}_{k+H}^\texttt{r})^{-1}} \le \min\left\{1,\|(A_i\varphi(\tilde{s}_h,\pi(\tilde{s}_h)))_{1\le i\le d}\|_{(\rG)^{-1}}\right\}$. Consequently 
\begin{equation*}
    \sum_{h=1}^H \|(A_i\varphi(\tilde{s}_h,\pi(\tilde{s}_h)))_{1\le i\le d}\|_{(\bar{G}_{k+H}^\texttt{r})^{-1}} \le \sqrt{H d\log(1+\alphar \eta^{-1} B_{\varphi, \bA}H)}.
\end{equation*}

\textit{3)} From \textit{1)} and \textit{2)}, we deduce that the total regret induced by rounds from $\mathcal{T}$ is bounded.
\begin{align}
    \sum_{k\in \mathcal{T}}&\sum_{h\in [H]} V_{\thetahp, \thetatr,1}^\pi(s_{1}^{k}) - V_{\thetahp, \thetar,1}^\pi(s_{1}^{k}) \nonumber \le \|\thetatr-\thetar\|_{\rG} \frac{\betar}{2}\\
    &\sqrt{\frac{3d}{\log(2)}\log\left(1+ \frac{\alphar\|\bA\|_2^2 B_{\varphi,\bA}^2}{\eta \log(2)}\right) \left(1+\frac{\alphar B_{\varphi, A} H}{\eta}\right) H d\log(1+\alphar \eta^{-1} B_{\varphi, \bA}H)} \label{eq:reward_estimation_bad_rounds}
\end{align}

\begin{remark}
The bad rounds analysis is one of our most important contributions as it enables us to forgo clipping without consequences. Consequently, this is a novel method to control the reward estimation error that improves on existing work for whom clipping was essential.
\end{remark}

\textbf{Good rounds.} Going forward we consider rounds from $\bar{\mathcal{T}}$. Let us define 
\begin{equation*}
    \zeta'_{k} \eqdef \operatorname{ traj}_{k} - \bE_{(\tilde{s}_h)_{1\le h\le H}\sim \pi \mid \thetahp, s_{1}^{k}} \left[\widetilde{\operatorname{traj}}_{k}\right].
\end{equation*}
where $\widetilde{\operatorname{traj}}_{k}$ is the same quantity as $\operatorname{traj}$ but with a random realization of state transitions.\\
Since all feature norms are smaller than one, $(\zeta'_{k})_{k}$ is a martingale sequence with $|\zeta'_{k}|\le \sqrt{H d \log \left(1+\alphar \eta^{-1} B_{\varphi, \mathbb{A}} H K\right)}$. We deduce that with probability at least $1-\delta$:
\begin{equation*}
    \sum_{k=1}^K \zeta'_{k} \le \sqrt{2K H d \log \left(1+\alphar \eta^{-1} B_{\varphi, \mathbb{A}} H K\right) \log(1/\delta)}
\end{equation*}

Therefore, we have with probability at least $1-3\delta$:
\begin{align*}
    \sum_{k\in \mathcal{T}^\texttt{c}} V_{\thetahp, \thetatr,1}^\pi (s_{1}^{k}) - V_{\thetahp, \thetar,1}^\pi(s_{1}^{k}) &\le \left(\sqrt{\frac{\betar(K,\delta)}{2\alphar}}+c\sqrt{(\max_{k} x_k)d\log(d K /\delta)}\right)\\
    &\times \betar \sqrt{K H d \log \left(1+\alphar \eta^{-1} B_{\varphi, \mathbb{A}} K H\right)\log(e/\delta^2)}.
\end{align*}

The last inequality follows from controlling the concentration of the reward parameter. First we observe that (Corollary~\ref{cor:Reward_euclidean_confidence_region}) with probability at least $1-\delta$, uniformly over $k \in \bN, \quad \left\|\thetar-\thetahr(k)\right\|_{\bar{G}_{k}^\texttt{r}}^{2} \le \frac{2}{\alphar} \betar(k,\delta)$. Second, we also have that for all $k\ge 1$, with probability at least $1-\delta, \|\xi_{ k}\|_{G_{k}^\texttt{r}} \leq c \sqrt{x_k d \log(d/\delta)}$, we then use a union bound. Combining with Equation~\eqref{eq:reward_estimation_bad_rounds} we find

\begin{align*}
    \sum_{k=1}^K V_{\thetahp, \thetatr,1}^\pi (s_{1}^{k}) - V_{\thetahp, \thetar,1}^\pi(s_{1}^{k}) &\le \left(\sqrt{\frac{\betar(K,\delta)}{2\alphar}}+c\sqrt{(\max_{k} x_k)d\log(d K /\delta)}\right)\\
    & \times \betar  \sqrt{K H d \log \left(1+\alphar \eta^{-1} B_{\varphi, \mathbb{A}} H K\right)\log(e/\delta^2)}.
\end{align*}
This concludes the proof.
\end{proof}

\begin{remark}
If we use Lemma~\ref{lem:Reward_Transportation} without the martingale difference sequence, it will lead to a linear regret. Indeed, the span of the sum of norms over an episode is of order $\sqrt{H}$. Using the martingale technique instead allows us to retrieve a telescopic sum controlled using the elliptical lemma, this is essential to obtaining a sub-linear regret bound. 
\end{remark}

\subsection{Learning error}\label{app:learning_error}

We now start the control of an important regret term, due to the distance between the estimated value function and the optimal value function.

\begin{lemma}\label{lem:learning_error}
    If the variance parameter of the injected noise $(\xi_k)_k$ satisfies
    \begin{equation*}
        x_k \ge \left(H\sqrt{\frac{\betap \betap(k,\delta)}{\alphap\alphar}}+\frac{\sqrt{\betar\betar(k,\delta)\min\{1,\frac{\alphap}{\alphar}\}}}{2 \alphar}\right),
    \end{equation*}
    then the learning error is controlled with probability at least $1-2\delta$ as
    \begin{align*}
        \sum_{k=1}^K V_1^\star (s_1^k) - V_{\thetahp,\thetahr+\bar{\xi}_k,1}^\pi (s_1^k) &\le \frac{d \betar \sqrt{x_k} \left(1 +\sqrt{ \log(d/\delta) }\right) }{\Phi(-1)} \sqrt{ H \log \left(1+\alphar \eta^{-1} B_{\varphi, \mathbb{A}} H K\right)}\\
        &\times \left( \sqrt{C_d \left(1+\frac{\alphar B_{\varphi, A} H}{\eta}\right)} + \sqrt{K \log(e/\delta^2)} \right),
    \end{align*}
    where for $\texttt{i}\in [\texttt{p}, \texttt{r}]$, $\beta^\texttt{i}(K,\delta)\triangleq \frac{\eta}{2}B_{\bA}^2+\gamma_K^\texttt{i}+\log(1/\delta)$, and $\gamma_K^\texttt{i}\triangleq d\log(1+\frac{\beta^\texttt{i}}{\eta}B_{\varphi,\bA}H K)$. Also $C_d \eqdef \frac{3d}{\log(2)}\log\left(1+ \frac{\alphar\|\bA\|_2^2 B_{\varphi,\bA}^2}{\eta \log(2)}\right)$, and $\Phi$ is the normal CDF.
\end{lemma}

This result basically means that we are no longer obliged to follow optimistic value functions, the perturbed estimation is enough to have a tight bound on the learning error.

\subsubsection{Stochastic optimism}\label{app:stochastic_optimism}

The goal here is to show that by injecting our carefully designed noise in the rewards we can ensure optimism with a constant probability. Consider the optimal policy $\pi^\star$, we have:
\begin{align*}
    (\tV{1}-\V{1}^\star &)(s_1) \ge (\tQ{1}^\star-Q_{1}^\star)(s_1,\pi^\star (s_1))\\
    &\ge \underbrace{V_{\thetahp,\thetar}^{\pi^\star} (s_1) - V_{\thetap,\thetar}^{\pi^\star} (s_1)}_{\text{first term}} + \underbrace{V_{\thetahp,\thetahr}^{\pi^\star} (s_1)  - V_{\thetahp,\thetar}^{\pi^\star} (s_1)}_{\text{second term}} + \underbrace{V_{\thetahp,\thetatr}^{\pi^\star} (s_1) - V_{\thetahp,\thetahr}^{\pi^\star} (s_1)}_{\text{third term}}
\end{align*}

\textbf{First term.} By assumption, the expected reward under the true parameter satisfies $\bE_{\thetar}[r(s,a)] \in [0,1]$, then $\bS\left(\sum_{t=1}^H \bE_{\thetar}[r(s_t,\pi(s_t))]\right)\le H$. Consequently, the first term can be controlled using Lemma~\ref{lem:transportation}
\begin{align*}
    &V_{\thetap,\thetar}^{\pi^\star} (s_1) - V_{\thetahp,\thetar}^{\pi^\star} (s_1) \le H \sqrt{\mathrm{KL}(P_{\thetahp}(s_2,\ldots, s_H),P_{\thetap}(s_2,\ldots, s_H)) }\\
    &\le H \sqrt{\bE_{(\tilde{s}_t)_{t \in [H]}\sim \thetahp \mid s_1^k}\left[\sum_{t=1}^H \psi(\tilde{s}_{t+1})^\top M_{\thetahp-\thetap} \varphi(\tilde{s}_t,\pi^\star (\tilde{s}_t)) + \pz_{\thetap}(\tilde{s}_t, \pi^\star (\tilde{s}_t)) - \pz_{\thetahp}(\tilde{s}_t, \pi^\star (\tilde{s}_t)) \right]}
\end{align*}

Using Taylor's expansion, for all $h\in [H], \exists \theta_h \in [\thetap,\thetahp]$ such that:
\begin{align*}
    \bE_{(\tilde{s}_t)_{t \in [H]}\sim \thetahp \mid s_1^k}&\left[ \psi(\tilde{s}_{t+1})^\top M_{\thetahp-\thetap} \varphi(\tilde{s}_t,\pi^\star (\tilde{s}_t)) + \pz_{\thetap}(\tilde{s}_t, \pi^\star (\tilde{s}_t)) - \pz_{\thetahp}(\tilde{s}_t, \pi^\star (\tilde{s}_t)) \right]\\
    &= \frac{1}{2} (\thetahp-\thetap)^\top \bE_{(\tilde{s}_t)_{t \in [H]}\sim \thetahp \mid s_1^k}\left[\nabla_{s_h,\pi^\star(s_h)}^2 \pz (\theta_h)\right] (\thetahp-\thetap)\\
    &\le \frac{\betap}{2} \bE_{(\tilde{s}_t)_{t \in [H]}\sim \thetahp \mid s_1^k}\left[\|\thetahp-\thetap\|_{G_{\tilde{s}_h, \pi^\star (\tilde{s}_h)}}^2\right].
\end{align*}
Define $u_k \eqdef \sum_{h=1}^H \bE_{(\tilde{s}_t)_{t \in [H]}\sim \thetahp \mid s_1^k} \left[ (A_i \varphi(\tilde{s}_h, \pi^\star (\tilde{s}_h)))_{i\in [d]}\right]$, then
\begin{align*}
    V_{\thetap,\thetar}^{\pi^\star} (s_1) - V_{ \thetahp, \thetar}^{\pi^\star} (s_1) &\le H \sqrt{ \frac{\betap}{2} \sum_{h=1}^H \bE_{(\tilde{s}_t)_{t \in [H]}\sim \thetahp \mid s_1^k}\left[\|\thetahp-\thetap\|_{G_{\tilde{s}_h, \pi^\star (\tilde{s}_h)}}^2\right]}\\
    &\le H \sqrt{\frac{\betap}{2}} \left\|\thetahp-\thetap\right\|_{\sum_{h=1}^H \bE_{(\tilde{s}_t)_{t \in [H]}\sim \thetahp \mid s_1^k} \left[ G_{\tilde{s}_h, \pi^\star (\tilde{s}_h)}\right]}  \\
    &\le H \sqrt{\frac{\betap}{2}} \left\|\thetahp-\thetap\right\|_{u_k u_k^\top} \\
    &\le H \sqrt{\frac{\betap}{2} \left\|(\pG)^{-1/2} u_k u_k^\top (\pG)^{-1/2}\right\|} \|\thetahp-\thetap\|_{\pG}\\
    &\le H \sqrt{\frac{\betap}{2}} \| u_k \|_{ (\pG)^{-1}} \|\thetahp-\thetap\|_{\pG}
\end{align*}
The third line follows because $\forall x\in \bR^d,\quad \|x\|_{\sum_{i=1}a_i a_i^\top} \le \|x\|_{(\sum_{i=1}a_i)(\sum_{i=1}a_i)^\top}$, and the last one follows because $\tr(AB)\le \tr(A)\tr(B)$ for any two real positive semi-definite matrices $A$ and $B$.\\
We deduce, with probability at least $1-\delta$:
\begin{equation*}
    V_{\thetap,\thetar}^{\pi^\star} (s_1) - V_{\thetahp,\thetar}^{\pi^\star} (s_1) \le H \sqrt{\frac{\betap \betap(k,\delta)}{\alphap}} \left\|\sum_{h=1}^H \bE_{(\tilde{s}_t)_{t \in [H]}\sim \thetahp \mid s_1^k} \left[ (A_i \varphi(\tilde{s}_h, \pi^\star (\tilde{s}_h)))_{i\in [d]}\right] \right\|_{(\pG)^{-1}}
\end{equation*}

\textbf{Second term.} We have
\begin{align*}
    V_{\thetahp,\thetahr}^{\pi^\star} (s_1)  - V_{\thetahp,\thetar}^{\pi^\star} (s_1) &= \bE_{(\tilde{s}_t)_{t \in [H]}\sim \thetahp \mid s_1^k}\left[\sum_{t=1}^H \frac{\Var^{\thetar_t} (r)}{2} B^\top M_{\thetahr -    \thetar}\varphi(\tilde{s}_t,\pi^\star (\tilde{s}_t))\right]\\
    &= (\thetahr - \thetar)^\top \bE_{(\tilde{s}_t)_{t \in [H]}\sim \thetahp \mid s_1^k}\left[\sum_{t=1}^H \frac{\Var^{\thetar_t} (r)}{2} (A_i\varphi(\tilde{s}_t,\pi^\star (\tilde{s}_t)))_{i\in [d]}\right] B\\
    &\le \frac{\sqrt{\betar}}{2} \|\thetahr - \thetar\|_{\rG} \left\| \bE_{(\tilde{s}_t)_{t \in [H]}\sim \thetahp \mid s_1^k}\left[\sum_{t=1}^H (A_i\varphi(\tilde{s}_t,\pi^\star (\tilde{s}_t)))_{i\in [d]}\right] \right\|_{(\rG)^{-1}}
\end{align*}
The last inequality comes from Cauchy-Schwarz. Applying that the norm (sum) makes appear only symmetric matrices times the variances so that we can bound the latter by $\betar$.\\
We conclude that with probability at least $1-\delta$,
\begin{equation*}
    V_{\thetahp,\thetahr}^{\pi^\star} (s_1)  - V_{\thetahp,\thetatr}^{\pi^\star} (s_1) \le \frac{\betar \sqrt{\betar(k,\delta)}}{\sqrt{2 \alphar}} \left\| \bE_{(\tilde{s}_t)_{t \in [H]}\sim \thetahp \mid s_1^k}\left[\sum_{t=1}^H (A_i\varphi(\tilde{s}_t,\pi^\star (\tilde{s}_t)))_{i\in [d]}\right] \right\|_{(\rG)^{-1}}
\end{equation*}
We want to write all the norms in the same matrix. Therefore, with probability at least $1-\delta$,
\begin{align*}
    V_{\thetahp,\thetahr}^{\pi^\star} (s_1)  - V_{\thetahp,\thetatr}^{\pi^\star} (s_1) \le& \sqrt{\frac{\betar \betar(k,\delta) \min\{1,\frac{\alphap}{\alphar}\}}{2 \alphar}}\\
    & \times \left\| \bE_{(\tilde{s}_t)_{t \in [H]}\sim \thetahp \mid s_1^k}\left[\sum_{t=1}^H (A_i\varphi(\tilde{s}_t,\pi^\star (\tilde{s}_t)))_{i\in [d]}\right] \right\|_{(\pG)^{-1}}
\end{align*}

\textbf{Third term.} We have
\begin{align*}
    V_{\thetahp,\thetahr,1}^{\pi^\star} (s_1)  - \tV{1}^{\pi^\star} (s_1) &= \bE_{(\tilde{s}_t)_{t \in [H]}\sim \thetahp \mid s_1^k} \left[\sum_{t=1}^H \frac{\Var^{\thetar_j} (r)}{2} B^\top M_{\thetahr -    \thetatr}\varphi(\tilde{s}_t,\pi^\star (\tilde{s}_t))\right]\\
    &= \xi_k ^\top \: \bE_{(\tilde{s}_t)_{t \in [H]}\sim \thetahp \mid s_1^k} \left[\sum_{t=1}^H \frac{\Var^{\thetar_j} (r)}{2} (A_i \varphi(\tilde{s}_t,\pi^\star (\tilde{s}_t)))_{i\in[d]}\right] B
\end{align*}

Given the normal CDF $\Phi$, we obtain that with probability at least $\Phi(-1)$
\begin{equation*}
    V_{\thetahp,\thetahr}^{\pi^\star} (s_1)  - V_{\thetahp,\thetatr}^{\pi^\star} (s_1) \ge \sqrt{x_k \alphar} \left\|\left[\sum_{t=1}^H \frac{\Var^{\thetar_j} (r)}{2} (A_i \varphi(\tilde{s}_t,\pi^\star (\tilde{s}_t)))_{i\in[d]}\right]\right\|_{(\pG)^{-1}}
\end{equation*}

Choosing $x_k \ge \left(H\sqrt{\frac{\betap \betap(k,\delta)}{\alphap\alphar}}+\frac{\sqrt{\betar\betar(k,\delta)\min\{1,\frac{\alphap}{\alphar}\}}}{2 \alphar}\right)$ and using Lemma~\ref{lem:gaussian_anti_concentration}, we find that the perturbed value function is optimistic with probability at least $\Phi(-1)$.

\subsubsection{Controlling the learning error}

In this section we see the core difference with optimistic algorithms. On the one hand, optimistic approaches require the value function generating the agent's policy to be larger than the optimal one with large probability, and can therefore ensure that the learning error is negative. On the other hand, \algo only ensures that the value function is optimistic with a constant probability: intuitively when this event holds the learning happens, and if it does not then the policy is still close to a good one thanks to the decreasing estimation error.

\paragraph{Upper bound on $V_1^\star$.} Let us draw $(\bar{\xi}_{k})_{k\in [K]}$ i.i.d copies of $(\xi_{k})_{k\in [K]}$. Define the optimism event at episode $k$:
\begin{equation}
    \label{def:optimism_k}
    \bar{O}_k = \{V_{\thetahp,\thetahr+\bar{\xi}_k,1}(s_1^k) - V_1^\star (s_1^k) \ge 0\}
\end{equation}
we know that $\bP(\bar{O}_k)\ge \Phi(-1)$. This event provides the upper bound:
\begin{equation}
    \label{eq:value_upper_bound}
    V_1^\star (s_1^k) \le \bE_{\bar{\xi}_k \mid \bar{O}_k }[V_{\thetahp,\thetahr+\bar{\xi}_k,1}(s_1^k)]
\end{equation}

\paragraph{Lower bound on $V_{\thetahp,\thetatr}$.} We define this bound with an optimization problem under concentration of the noise. Consider $\underbar{V}_{1}(s_1^k)$ is the solution of
\begin{align}
    &\min _{\xi_{ k}} V_{\thetahp, \thetahr+\xi_k,1}(s_{1^k}) \label{eq:value_lower_bound_opt}\\
    &\|\xi_k\|_{\pG} \le \sqrt{x_k d\log(d/\delta)}, \quad \forall t \in[H] \nonumber
\end{align}
Under the concentration of our injected noise, we obtain
\begin{equation}
    \label{eq:value_lower_bound}
    \underbar{V}_{1}(s_1^k) \le V_{\thetahp,\thetatr}(s_1^k)
\end{equation}

\paragraph{Combining the error bounds.} Combining the upper bound of Equation~\eqref{eq:value_upper_bound} with the lower bound of Equation~\eqref{eq:value_lower_bound}, we get, with probability at least $1-\delta$:
\begin{equation*}
    V_1^\star (s_1^k) - V_{\thetahp,\thetahr+\bar{\xi}_k,1}(s_1^k) \le \bE_{\bar{\xi}_k \mid \bar{O}_k }[V_{\thetahp,\thetahr+\bar{\xi}_k,1}(s_1^k)-\underbar{V}_{1}(s_1^k)]
\end{equation*}
Also, using the tower rule,
\begin{align*}
    \bE_{\bar{\xi}_k }[&V_{\thetahp,\thetahr+\bar{\xi}_k,1}(s_1^k)-\underbar{V}_{1}(s_1^k)]\\
    &= \bE_{\bar{\xi}_k \mid \bar{O}_k }[V_{\thetahp,\thetahr+\bar{\xi}_k,1}(s_1^k)-\underbar{V}_{1}(s_1^k)]\bP(\bar{O}_k) + \bE_{\bar{\xi}_k \mid \bar{O}_k^\texttt{c} }[V_{\thetahp,\thetahr+\bar{\xi}_k,1}(s_1^k)-\underbar{V}_{1}(s_1^k)]\bP(\bar{O}_k^\texttt{c})
\end{align*}
Therefore,
\begin{align*}
    V_1^\star (s_1^k) - &V_{\thetahp,\thetahr+\bar{\xi}_k,1}(s_1^k)\\
    &\le \left(\bE_{\bar{\xi}_k }[V_{\thetahp,\thetahr+\bar{\xi}_k,1}(s_1^k)-\underbar{V}_{1}(s_1^k)] - \bE_{\bar{\xi}_k \mid \bar{O}_k^\texttt{c} }[V_{\thetahp,\thetahr+\bar{\xi}_k,1}(s_1^k)-\underbar{V}_{1}(s_1^k)]\bP(\bar{O}_k^\texttt{c})\right)/ \bP(\bar{O}_k)\\
    &=\left(\bE_{\xi_k }[V_{\thetahp,\thetahr+{\xi}_k,1}^\pi (s_1^k)-\underbar{V}_{1}^\pi (s_1^k)] - \bE_{\xi_k \mid \bar{O}_k^\texttt{c} }[V_{\thetahp,\thetahr+\xi_k,1}(s_1^k)-\underbar{V}_{1}(s_1^k)]\bP(\bar{O}_k^\texttt{c})\right)/ \bP(\bar{O}_k).
\end{align*}
The last line follows since $\xi_k$ and $\bar{\xi}_k$ are i.i.d.

The rest of the analysis proceeds similarly to the proof of the reward estimation. 

Let us call the argument of the minimum in Equation~\eqref{eq:value_lower_bound_opt} as $\underline{\xi}_k$. Using Lemma~\ref{lem:Reward_Transportation}, we find
\begin{align*}
    \tV{1}^\pi(s_{1}^{k}) - &V_{\thetahp, \thetahr+\underline{\xi}_k,1}^\pi(s_{1}^{k})\\
    &= \bE_{(\tilde{s}_h)_{1\le h\le H}\sim \pi \mid \thetahp, s_{1}^{k}}\left[\sum_{h=1}^H \frac{\Var_{\tilde{s}_h,\pi(\tilde{s}_h)}(r)}{2} B^\top M_{\thetatr-\thetahr-\underline{\xi}_k}\varphi(\tilde{s}_h,\pi(\tilde{s}_h))\right] \\
    &\le \bE\left[\sum_{h=1}^H \frac{\Var_{\tilde{s}_h,\pi(\tilde{s}_h)}(r)}{2} \|\thetatr-\thetahr-\underline{\xi}_k\|_{\pG} \|(B^\top A_i\varphi(\tilde{s}_h,\pi(\tilde{s}_h)))_{1\le i\le d}\|_{(\pG)^{-1}}\right]\\
    &\le \|\thetatr-\thetahr-\underline{\xi}_k\|_{\pG}  \bE\left[\sum_{h=1}^H \frac{\Var_{\tilde{s}_h,\pi(\tilde{s}_h)}(r)}{2} \|(B^\top A_i\varphi(\tilde{s}_h,\pi(\tilde{s}_h)))_{1\le i\le d}\|_{(\pG)^{-1}}\right]\\
    &\le \|\tilde{\xi}_k-\underline{\xi}_k\|_{\pG} \frac{\betar}{2}  \bE \Bigg[ \sum_{h=1}^H \|(A_i\varphi(\tilde{s}_h,\pi(\tilde{s}_h)))_{1\le i\le d}\|_{(\pG)^{-1}} \Bigg]
\end{align*}
Then,
\begin{align*}
    \bE_{\tilde{\xi}_k} \Big[\tV{1}^\pi(s_{1}^{k}) - &V_{\thetahp, \thetahr+\underline{\xi}_k,1}^\pi (s_{1}^{k}) \Big]\\
    &\le \frac{\betar}{2} \bE_{\tilde{\xi}_k}[\|\tilde{\xi}_k-\underline{\xi}_k\|_{\pG}] \bE_{(\tilde{s}_h) \sim \pi \mid \thetahp}\left[\sum_{h=1}^H \|(A_i\varphi(\tilde{s}_h,\pi(\tilde{s}_h)))_{1\le i\le d}\|_{(\pG)^{-1}}\right].
\end{align*}
Also,
\begin{align*}
    \Big|\bE_{\xi_k \mid \bar{O}_k^\texttt{c} }[V_{\thetahp,\thetahr+\xi_k,1}(s_1^k)-&\underline{V}_{1}(s_1^k)]\Big|\\
    \le \frac{\betar}{2}& \bE_{\tilde{\xi}_k \mid \bar{O}_k^\texttt{c}}[\|\tilde{\xi}_k-\underline{\xi}_k\|_{\pG}] \bE_{(\tilde{s}_h) \sim \pi \mid \thetahp}\left[\sum_{h=1}^H \|(A_i\varphi(\tilde{s}_h,\pi(\tilde{s}_h)))_{1\le i\le d}\|_{(\pG)^{-1}}\right]\\
    \le \frac{\betar}{2}& \bE_{\tilde{\xi}_k}[\|\tilde{\xi}_k-\underline{\xi}_k\|_{\pG}] \bE_{(\tilde{s}_h) \sim \pi \mid \thetahp}\left[\sum_{h=1}^H \|(A_i\varphi(\tilde{s}_h,\pi(\tilde{s}_h)))_{1\le i\le d}\|_{(\pG)^{-1}}\right].
\end{align*}
We have a bound on the expected value of the sum of feature norms in the proof of Lemma~\ref{lem:reward_estimation_bound}. Also,
\begin{align*}
    \bE_{\tilde{\xi}_k}[\|\tilde{\xi}_k-\underline{\xi}_k\|_{\pG}] &\le \bE_{\tilde{\xi}_k}[\|\tilde{\xi}_k\|_{\pG}] + \bE_{\tilde{\xi}_k}[\|\underline{\xi}_k\|_{\pG}]\\
    &\le \sqrt{\bE_{\tilde{\xi}_k}[\|\tilde{\xi}_k\|_{\pG}^2]} + \sqrt{x_k d \log(d/\delta)}\\
    &\le \sqrt{x_k d}+\sqrt{x_k d \log(d/\delta)}
\end{align*}
The second line follows from Cauchy-Schwarz and by definition of $\underline{\xi}_k$. The last line is due to the fact that $x_k(\pG)^{-1} \sim \cN(0,x_k I_d)$, which implies $\|\tilde{\xi}_k\|_{\pG}^2 \sim \cN(0, d x_k)$. We conclude the proof by taking the sum of feature norms from the proof of Lemma~\ref{lem:reward_estimation_bound}.

We conclude that with probability at least $1-2\delta$:
\begin{align*}
    \sum_{k=1}^K V_1^\star &(s_1^k) - V_{\thetahp,\thetahr+\bar{\xi}_k,1}(s_1^k)\le \frac{\betar}{\Phi(-1)}(\sqrt{x_k d}+\sqrt{x_k d \log(d/\delta)})\\
    &\Bigg[ \sqrt{\frac{3d}{\log(2)}\log\left(1+ \frac{\alphar\|\bA\|_2^2 B_{\varphi,\bA}^2}{\eta \log(2)}\right) \left(1+\frac{\alphar B_{\varphi, A} H}{\eta}\right) H d\log(1+\alphar \eta^{-1} B_{\varphi, \bA}H)} \\
    &+ \sqrt{K H d \log \left(1+\alphar \eta^{-1} B_{\varphi, \mathbb{A}} H K\right)\log(e/\delta^2)} \Bigg]
\end{align*}

\section{Concentrations}\label{app:concentrations}

\subsection{Concentration of the transition parameter} \label{app:transition_concentration}

We recall the important concentration of the maximum likelihood estimator for general bilinear exponential families (\cf Theorem 1 of \cite{chowdhury2021reinforcement}).
\begin{theorem} 
    \label{thm:Transition_confidence_region}
    Suppose $\left\{\mathcal{F}_{t}\right\}_{t=0}^{\infty}$ is a filtration such that for each $t$, (i) $s_{t+1}$ is $\mathcal{F}_{t}$-measurable, (ii) $\left(s_{t}, a_{t}\right)$ is $\mathcal{F}_{t-1}$ measurable, and (iii) given $\left(s_{t}, a_{t}\right), s_{t+1} \sim P_{\thetap}^\texttt{p} \left(\cdot \mid s_{t}, a_{t}\right)$ according to the exponential family defined by Equation~\eqref{def:transition_model}. Let $\thetahp(k)$ be the penalized MLE defined by Equation~\eqref{eq:ML_transition}, and let $\pz_{s, a}(\theta)$ be strictly convex in $\theta$ for all $(s, a).$ Then, for any $\delta \in(0,1]$, with probability at least $1-\delta$, the following holds uniformly over all $n \in \mathbb{N}$ :
    \begin{equation*}
        \sum_{t=1}^{k} \mathrm{KL}_{s_{t}, a_{t}}\left(\thetahp(k), \thetap\right)+\frac{\eta}{2}\left\|\thetap-\thetahp(k)\right\|_{\mathbb{A}}^{2}-\frac{\eta}{2}\left\|\thetap\right\|_{\mathbb{A}}^{2} \leq \log \left(\frac{C_{\mathrm{A}, k}^\texttt{p}}{\delta}\right),
    \end{equation*}
    where $C_{\mathrm{A}, k}^\texttt{p} = \left(\!\int_{\bR^{d}} \exp\! \left(\!-\!\frac{\eta}{2}\left\|\theta^{\prime}\right\|_{\mathbb{A}}^{2}\right)\! d \theta^{\prime}\!\right)/\left(\!\int_{\bR^d} \exp \left(-\sum_{t=1}^{k} \mathrm{KL}_{s_{t}, a_{t}}\left(\theta_{k}, \theta^{\prime}\right)\!-\!\frac{\eta}{2}\left\|\theta^{\prime}-\theta_{k}\right\|_{\mathbb{A}}^{2}\right)\! d \theta^{\prime}\!\right)$. Define $G_{s, a}\eqdef \left(\varphi(s, a)^{\top} A_{i}^{\top} A_{j} \varphi(s, a)\right)_{i, j \in [d]}$, we have
    \begin{equation*}
        C_{\mathbb{A}, k}^\texttt{p} \leq \det\left(I+\betap \eta^{-1} \mathbb{A}^{-1} \sum_{t=1}^{k} G_{s_{t}, a_{t}}\right),
    \end{equation*}
    where $\betap=\sup_{\theta, s, a} \lambda_{ \max } \left(\mathbb{C}_{s, a}^{ \theta }\left[ \psi \left( s^{\prime} \right) \right]\right)$.
\end{theorem}

A proof of this result can be found in the work \cite{chowdhury2021reinforcement}. We provide an almost similar proof for the concentration of rewards in the next section.

\begin{corollary}
\label{cor:Transition_euclidean_confidence_region}
The previous theorem implies a simple euclidean confidence region. Indeed, with probability at least $1-\delta$, for all $k\in \bN$
\begin{equation*}
    \left\|\thetap-\thetahp(k)\right\|_{\bar{G}_{n}^\texttt{p}}^{2} \le \frac{2}{\alphap} \betap(k,\delta),
\end{equation*}
where $\betap(k,\delta) \eqdef \betap_{(k-1) H}(\delta)=\frac{2}{2} B_{A}^{2}+\log \left(2 C_{A,k}^\texttt{p} / \delta\right)$.
\end{corollary}
\begin{proof}
The result follows from the following simple calculations:
\begin{align*}
    \frac{1}{2}\left\|\thetap-\thetahp(k)\right\|_{\bar{G}_{k}}^{2} &= \frac{(\alphap)^{-1} \eta}{2}\left\|\thetap-\thetahp(k)\right\|_{\mathbb{A}}^{2}+\sum_{\tau=1}^{k-1} \sum_{h=1}^{H} \frac{1}{2}\left\|\thetap-\thetahp(k)\right\|_{G_{s_h^\tau, a_{h}^{\tau}}}^{2}\\
    &\leq (\alphap)^{-1}\left(\frac{\eta}{2}\left\|\thetap-\thetahp(k)\right\|_{\mathbb{A}}^{2}+\sum_{\tau=1}^{k-1} \sum_{h=1}^{H} \mathrm{KL}_{s_{h}^{\tau}, a_{h}^{\tau}}\left(\theta_{k}, \theta \right) \right).
\end{align*}
\end{proof}

\subsection{Concentration of the reward parameter (contribution)}\label{app:reward_concentration}

\begin{theorem}\label{thm:Reward_confidence_region}
    Suppose $\left\{\mathcal{F}_{t}\right\}_{t=0}^{\infty}$ is a filtration such that for each $t$, (i) $r(s_t,a_t)$ is $\mathcal{F}_{t}$-measurable, (ii) $\left(s_{t}, a_{t}\right)$ is $\mathcal{F}_{t-1}$ measurable, and (iii) given $\left(s_{t}, a_{t}\right), r(s_t,a_t) \sim P_{\thetar}^\texttt{r} \left(\cdot \mid s_{t}, a_{t}\right)$ according to the exponential family defined by \eqref{def:reward_model}. Let $\thetahr(k)$ be the penalized MLE defined by Equation~\eqref{eq:ML_reward}, and let $\rz_{s, a}(\theta)$ be strictly convex in $\theta$ for all $(s, a).$ Then, for any $\delta \in(0,1]$, with probability at least $1-\delta$, the following holds uniformly over all $k \in \mathbb{N}$ :
    \begin{equation*}
        \sum_{t=1}^{k} \mathrm{KL}_{s_{t}, a_{t}}\left(\thetahr(k), \thetar\right)+\frac{\eta}{2}\left\|\thetar-\thetahr(k)\right\|_{\mathbb{A}}^{2}-\frac{\eta}{2}\left\|\thetar\right\|_{\mathbb{A}}^{2} \leq \log \left(\frac{C_{\mathrm{A}, k}^\texttt{r}}{\delta}\right),
    \end{equation*}
    where $C_{\mathrm{A}, k}^\texttt{r}=\left(\!\int_{\bR^{d}} \exp\! \left(\!-\!\frac{\eta}{2}\left\|\theta^{\prime}\right\|_{\mathbb{A}}^{2}\right)\! d \theta^{\prime}\!\right)/\left(\!\int_{\bR^d} \exp \left(-\sum_{t=1}^{k} \mathrm{KL}_{s_{t}, a_{t}}\left(\theta_{k}, \theta^{\prime}\right)\!-\!\frac{\eta}{2}\left\|\theta^{\prime}-\theta_{k}\right\|_{\mathbb{A}}^{2}\right)\! d \theta^{\prime}\!\right)$. Define $G_{s, a}\eqdef \left(\varphi(s, a)^{\top} A_{i}^{\top} A_{j} \varphi(s, a)\right)_{i, j \in [d]}$, we have
    \begin{equation*}
        C_{\mathbb{A}, k} \leq \det\left(I+\betar \eta^{-1} \mathbb{A}^{-1} \sum_{t=1}^{k} G_{s_{t}, a_{t}}\right),
    \end{equation*}
    where $\betar:=\|B\|_2^2 \: \sup _{\theta, s, a}  \Var_{s,a}^\theta(r)$.
\end{theorem}

\begin{proof}
We proceed similar to the proof of Theorem 1 in \cite{chowdhury2019online}.

\paragraph{Step 1: Martingale construction.} First, observe that by assuming strict convexity, the log-partition function $\rz_{s, a}$ becomes a Legendre function. Now for the conditional exponential family model, the KL divergence between $\probr{\cdot \mid s, a}$ and $\probrp{\cdot \mid s, a}$ can be expressed as a Bregman divergence associated to $\rz_{s, a}$ with the parameters reversed, i.e.
\begin{equation*}
    \mathrm{KL}_{s, a}\left(\thetar, \thetarp \right):=\mathrm{KL}\left(P_{\thetar}(\cdot \mid s, a), P_{\thetarp}(\cdot \mid s, a)\right)=B_{Z_{s, a}}\left(\thetarp, \thetar\right).
\end{equation*}
Now, for any $\lambda \in \mathbb{R}^{d}$, we introduce the function $B_{Z_{n, \alpha}, \thetar}(\lambda)=B_{Z_{n, \alpha}}\left(\thetar+\lambda, \lambda\right)$ and define
\begin{equation*}
    M_{n}^{\lambda}=\exp \left(\lambda^{\top} S_{n}-\sum_{t=1}^{n} B_{Z_{n_{t}, a_{t}}, \thetar}(\lambda)\right)
\end{equation*}
where $\forall i \leq d$, we denote $\left(S_{n}\right)_{i}=\sum_{t=1}^{n}\left(r\left(s_{t},a_t \right)-\mathbb{E}_{s_{t}, a_{t}}^{\thetar}\left[r\right]\right) B^{\top} A_{i} \varphi\left(s_{t}, a_{t}\right) .$ Note that $M_{n}^{\lambda}>0$ and it is $\mathcal{F}_{n^{-}}$ measurable. Furthermore, we have for all $(s, a)$,
\begin{align*}
    \mathbb{E}_{s, a}^{\thetar}&\left[\exp \left(\sum_{i=1}^{d} \lambda_{i}\left(r\left(s_{t},a_t \right)-\mathbb{E}_{s_{t}, a_{t}}^{\thetar}\left[r\right]\right) B^{\top} A_{i} \varphi\left(s_{t}, a_{t}\right)\right)\right] \\
    &=\exp \left(-\lambda^{\top} \nabla \rz_{s, a}\left(\thetar\right)\right) \int_{\mathcal{S}} \exp \left(\sum_{i=1}^{d}\left(\thetar_{i}+\lambda_{i}\right) B^\top A_{i} \varphi(s, a)-\rz_{s, a}(\thetar)\right) d r \\
    &=\exp \left(\rz_{s, a}(\thetar+\lambda)-\rz_{s, a}(\thetar)-\lambda^{\top} \nabla \rz_{s, a}(\thetar)\right)=\exp \left(B_{\rz_{s, a}}(\thetar)\right)
\end{align*}
This implies $\mathbb{E}\left[\exp \left(\lambda^{\top} S_{n}\right) \mid \mathcal{F}_{n-1}\right]=\exp \left(\lambda^{\top} S_{n-1}+B_{Z_{n_{n}, a_{n}, \thetar}}(\lambda)\right)$ thus $\mathbb{E}\left[M_{n}^{\lambda} \mid \mathcal{F}_{n-1}\right]=M_{n-1}^{\lambda}$. Therefore $\left\{M_{n}^{\lambda}\right\}_{n=0}^{\infty}$ is a non-negative martingale adapted to the filtration $\left\{\mathcal{F}_{n}\right\}_{n=0}^{\infty}$ and actually satisfies $\mathbb{E}\left[M_{n}^{\lambda}\right]=1$. For any prior density $q(\theta)$ for $\theta$, we now define a mixture of martingales
\begin{equation}
    \label{def:mixture_martingales}
    M_{n}=\int_{\mathbb{R}^{d}} M_{n}^{\lambda} q\left(\thetar+\lambda\right) d \lambda
\end{equation}
Then $\left\{M_{n}\right\}_{n=0}^{\infty}$ is also a non-negative martingale adapted to $\left\{\mathcal{F}_{n}\right\}_{n=0}^{\infty}$ and in fact, $\mathbb{E}\left[M_{n}\right]=1$.

\paragraph{Step 2: Method of mixtures.} Considering the prior density $\cN(0,(\eta \bA)^{-1})$, we obtain from \eqref{def:mixture_martingales} that
\begin{equation}
    \label{eq:mixture_martingales}
    M_{n}=c_{0} \int_{\mathbb{R}^{d}} \exp \left(\lambda^{\top} S_{n}-\sum_{t=1}^{n} B_{\rz_{x_{t}, a_{t}}, \thetar}(\lambda)-\frac{\eta}{2}\left\|\thetar+\lambda\right\|_{\bA}^{2}\right) d \lambda,
\end{equation}
where $c_{0}=\frac{1}{\int_{\bR^{d}} \exp \left(-\frac{\eta}{2}\left\|\theta^{\prime}\right\|_{\Lambda}^{2}\right) d \theta^{\prime}} .$ We now introduce the function $\rz_{n}(\theta)=\sum_{t=1}^{n} \rz_{s_{t}, a_{t}}(\theta) .$ Note that $\rz_{n}$ is a also Legendre function and its associated Bregman divergence satisfies
\begin{equation*}
    B_{\rz_{n}}\left(\theta^{\prime}, \theta\right)=\sum_{t=1}^{n}\left(\rz_{s_{t}, a_{t}}\left(\theta^{\prime}\right)-\rz_{s_{t}, a_{t}}(\theta)-\left(\theta^{\prime}-\theta\right)^{\top} \nabla \rz_{S_{t}, a_{t}}(\theta)\right)=\sum_{t=1}^{n} B_{\rz_{s_{t}, \alpha_{t}}}\left(\theta^{\prime}, \theta\right)
\end{equation*}
Furthermore, we have $\sum_{t=1}^{n} B_{\rz_{s_{t}, \alpha_{t}}, \thetar}(\lambda)=B_{\rz_{n}, \thetar}(\lambda)$. From the penalized likelihood formula \eqref{eq:ML_reward}, recall that
\begin{equation*}
    \forall i \leq d, \quad \sum_{t=1}^{n} \nabla_{i} \rz_{s_{t}, a_{t}}\left(\thetahr(k)\right)+\frac{\eta}{2} \nabla_{i}\|\thetahr(k)\|_{\bA}^{2}=\sum_{t=1}^{k} r_t B^{\top} A_{i} \varphi\left(s_{t}, a_{t}\right).
\end{equation*}
This yields
\begin{equation}
    \label{eq:S_n}
    S_{k}=\sum_{t=1}^{k}\left(\nabla \rz_{s_{t}, a_{t}}\left(\thetahr(k)\right)-\nabla \rz_{s_{t}, a_{t}}\left(\thetar\right)\right)+\eta \bA \thetahr(k)=\nabla \rz_{k}\left(\thetahr(k)\right)-\nabla \rz_{k}\left(\thetar\right)+\eta \bA \thetahr(k)
\end{equation}
We now obtain from \eqref{eq:mixture_martingales} and \eqref{eq:S_n} that
\begin{equation}
    \label{eq:9}
    M_{k}=c_{0} \cdot \exp \left(-\frac{\eta}{2}\left\|\thetar\right\|_{A}^{2}\right) \int_{\mathbb{R}^{d}} \exp \left(\lambda^{\top} x_{k}-B_{Z_{k}, \theta^{*}}(\lambda)+g_{k}(\lambda)\right) d \lambda,
\end{equation}
where we introduced $g_{k}(\lambda)=\frac{\eta}{2}\left(2 \lambda^{\top} \bA \thetahr(k)+\left\|\thetar\right\|_{\bA}^{2}-\left\|\thetar+\lambda\right\|_{\bA}^{2}\right)$ and $x_{k}=\nabla \rz_{k}\left(\thetahr(k)\right)-\nabla \rz_{k}\left(\thetar\right)$.\\
Now, note that $\sup _{\lambda \in \mathbb{R}^{d}} g_{k}(\lambda)=\frac{\eta}{2}\left\|\thetar-\thetahr(k)\right\|_{\bA}^{2}$, where the supremum is attained at $\lambda^{\star}=\thetahr(k)-\thetar$. We then have
\begin{align}
    g_{k}(\lambda) &=g_{n}(\lambda)+\sup _{\lambda \in \mathbb{R}^{\star}} g_{k}(\lambda)-g_{k}\left(\lambda^{\star}\right) \nonumber\\
    &=\frac{\eta}{2}\left\|\thetahr(k)-\thetar\right\|_{\bA}^{2}+\eta\left(\lambda-\lambda^{\star}\right)^{\top} \bA\left(\thetar+\lambda^{\star}\right)+\frac{\eta}{2}\left\|\thetar+\lambda^{\star}\right\|_{A}^{2}-\frac{\eta}{2}\left\|\thetar+\lambda\right\|_{\bA}^{2} \nonumber\\
    &=B_{\rz_{0}}\left(\thetar, \thetahr(k)\right)+\left(\lambda-\lambda^{\star}\right)^{\top} \nabla \rz_{0}\left(\thetar+\lambda^{\star}\right)+\rz_{0}\left(\thetar+\lambda^{\star}\right)-\rz_{0}\left(\thetar+\lambda\right) \label{eq:10}
\end{align}
where we have introduced the Legendre function $\rz_{0}(\theta)=\frac{\eta}{2}\|\theta\|_{\bA}^{2}$. We now have from \eqref{eq:Bregman_duality} that
\begin{align*}
    \sup _{\lambda \in \mathbb{R}^{d}}&\left(\lambda^{\top} x_{n}-B_{\rz_{n}, \thetar}(\lambda)\right) \\
    &=B_{\rz_{n}, \thetar}^{\star}\left(x_{n}\right)=B_{\rz_{n}, \thetar}^{\star}\left(\nabla \rz_{n}\left(\thetahr(n)\right)-\nabla \rz_{n}\left(\thetar\right)\right)=B_{\rz{n}}\left(\thetar, \thetahr(n)\right).
\end{align*}
Further, any optimal $\lambda$ must satisfy
\begin{equation*}
    \nabla \rz_{n}\left(\thetar+\lambda\right)-\nabla \rz_{n}\left(\thetar\right)=x_{n} \Longrightarrow \nabla \rz_{n}\left(\thetar+\lambda\right)=\nabla \rz_{n}\left(\thetahr(n)\right).
\end{equation*}
One possible solution is $\lambda=\lambda^{\star}$. Now, since $\rz_{n}$ is strictly convex, the supremum is indeed attained at $\lambda=\lambda^{\star}$. We then have
\begin{align}
    \lambda^{\top} x_{n} &- B_{\rz_{n}, \thetar}(\lambda) \nonumber\\
    &=\lambda^{\top} x_{n}-B_{\rz_{n}, \thetar}(\lambda)+B_{\rz_{n}}\left(\thetar, \thetahr(n)\right)-\left(\lambda^{\star} x_{n}-B_{\rz_{n}, \thetar}\left(\lambda^{\star}\right)\right) \nonumber\\
    &=B_{\rz_{n}}\left(\thetar, \thetahr(n)\right)+\left(\lambda-\lambda^{\star}\right)^{\top} \nabla \rz_{n}\left(\thetar+\lambda^{\star}\right)+B_{\rz_{n}, \theta^{*}}\left(\lambda^{\star}\right)-B_{\rz_{n}, \theta^{*}}(\lambda) \nonumber\\
    & \quad-\left(\lambda-\lambda^{\star}\right)^{\top} \nabla \rz_{n}\left(\thetar\right) \nonumber\\
    &=B_{\rz_{n}}\left(\thetar, \thetahr(n)\right)+\left(\lambda-\lambda^{\star}\right)^{\top} \nabla \rz_{n}\left(\thetar+\lambda^{\star}\right)+\rz_{n}\left(\thetar+\lambda^{\star}\right)-\rz_{n}\left(\thetar+\lambda\right) \label{eq:11}
\end{align}

Plugging Equation~\eqref{eq:10} and Equation~\eqref{eq:11} in Equation~\eqref{eq:9}, we obtain
\begin{align*}
    M_{n} &= c_{0} \cdot \exp \left(\sum_{j \in\{0, n\}} B_{\rz_{j}}\left(\thetar, \theta_{j}\right)-\frac{\eta}{2}\left\|\thetar\right\|_{A}^{2}\right) \\
    & \qquad \times \int_{\mathbb{R}^{d}} \exp \left(\sum_{j \in\{0, n\}}\left(\left(\lambda-\lambda^{\star}\right)^{\top} \nabla \rz_{j}\left(\thetar+\lambda^{\star}\right)+\rz_{j}\left(\thetar+\lambda^{\star}\right)-\rz_{j}\left(\thetar+\lambda\right)\right)\right) d \lambda \\
    &= c_{0} \cdot \exp \left(\sum_{j \in\{0, n\}} B_{\rz_{j}}\left(\thetar, \thetahr(n)\right)-\frac{\eta}{2}\left\|\thetar\right\|_{\AA}^{2}\right)\\
    & \qquad \times \exp \left(-\sum_{j \in\{0, n\}}\left(\left(\thetar+\lambda^{\star}\right)^{\top} \nabla \rz_{j}\left(\thetar+\lambda^{\star}\right)-\rz_{j}\left(\thetar+\lambda^{\star}\right)\right)\right) \\
    & \qquad \times \int_{\mathbb{R}^{d}} \exp \left(\sum_{j \in\{0, n\}}\left(\left(\thetar+\lambda\right)^{\top} \nabla \rz_{j}\left(\thetar+\lambda^{\star}\right)-\rz_{j}\left(\thetar+\lambda\right)\right)\right) d \lambda \\
    &= \frac{c_{0}}{c_{\mathrm{n}}} \exp \left(\sum_{j \in\{0, n\}} B_{\rz_{j}}\left(\thetar, \thetahr(n)\right)-\frac{\eta}{2}\left\|\thetar\right\|_{\mathbb{A}}^{2}\right)\\
    &\qquad\times \frac{\int_{\mathbb{R}^{d}} \exp \left(\sum_{j \in\{0, n\}}\left(\left(\thetar+\lambda\right)^{\top} \nabla \rz_{j}\left(\thetar+\lambda^{\star}\right)-\rz_{j}\left(\thetar+\lambda\right)\right)\right) d \lambda}{\int_{\mathbb{R}^{d}} \exp \left(\sum_{j \in\{0, n\}}\left(\left(\theta^{\prime}\right)^{\top} \nabla \rz_{j}\left(\thetar+\lambda^{\star}\right)-\rz_{j}\left(\theta^{\prime}\right)\right)\right) d \theta^{\prime}} \\
    &= \frac{c_{0}}{c_{n}} \cdot \exp \left(B_{Z_{n}}\left(\thetar, \thetahr(n)\right)+B_{Z_{0}}\left(\thetar, \thetahr(n)\right)-\frac{\eta}{2}\left\| \thetar \right\|_{\mathbb{A}}^{2}\right),
\end{align*}

where we introduced $c_{n}=\frac{\exp \left(\sum_{j \in\{0, n\}}\left(\left(\thetar+\lambda^{*}\right)^{\top} \nabla \rz_{j}\left(\thetar+\lambda^{*}\right)-\rz_{j}\left(\thetar+\lambda^{*}\right)\right)\right)}{\int_{\bR^{d}} \exp \left(\sum_{j \in\{0, n\}}\left(\left(\theta^{\prime}\right)^{\top} \nabla \rz_{j}\left(\thetar+\lambda^{*}\right)-\rz_{j}\left(\theta^{\prime}\right)\right)\right) d \theta^{\prime}} .$ Since $\lambda^{\star}=\thetahr(n)-\thetar$, we have $$ c_{n}=\frac{1}{\int_{\mathbb{R}^{d}} \exp \left(-\sum_{j \in\{0, n\}} B_{\rz_{j}}\left(\theta^{\prime}, \thetar+\lambda^{\star}\right)\right) d \theta^{\prime}}=\frac{1}{\int_{\mathbb{R}^{d}} \exp \left(-\sum_{t=1}^{n} B_{Z_{s_{t}, a_{t}}}\left(\theta^{\prime}, \thetahr(n)\right)-\frac{\eta}{2}\left\|\theta^{\prime}-\thetahr(n)\right\|_{\mathbb{A}^{\prime}}^{2}\right) d \theta^{\prime}} $$ Therefore, we have from $(5)$ that $$ C_{A, n}:=\frac{c_{n}}{c_{0}}=\frac{\int_{\mathbb{R}^{d}} \exp \left(-\frac{\eta}{2}\left\|\theta^{\prime}\right\|_{\mathbb{A}}^{2}\right) d \theta^{\prime}}{\int_{\mathbb{R}^{d}} \exp \left(-\sum_{t=1}^{n} \mathrm{KL}_{s_{t}, a_{t}}\left(\thetahr(n), \theta^{\prime}\right)-\frac{\eta}{2}\left\|\theta^{\prime}-\thetahr(n)\right\|_{\mathbb{A}}^{2}\right) d \theta^{\prime}} $$ An application of Markov's inequality now yields
\begin{equation*}
    \mathbb{P}\left[\sum_{t=1}^{n} \mathrm{KL}_{s_{t}, a_{t}}\left(\thetahr(n), \thetar\right) \! + \! \frac{\eta}{2}\left\|\thetar-\thetahr(n)\right\|_{\mathbb{A}}^{2} \! - \! \frac{\eta}{2}\left\|\thetar\right\|_{\mathbb{A}}^{2} \geq \! \log \! \left( \! \frac{C_{A, n}}{\delta} \! \right) \! \right] \! \!= \! \mathbb{P}\left[M_{n} \geq \frac{1}{\delta}\right] \leq \delta \mathbb{E}\left[M_{n}\right] \!= \! \delta 
\end{equation*}
\paragraph{Step 3: A stopped martingale and its control.} Let $N$ be a stopping time with respect to the filtration $\left\{\mathcal{F}_{n}\right\}_{n=0}^{\infty}$. Now, by the martingale convergence theorem, $M_{\infty}=\lim _{n \rightarrow \infty} M_{n}$ is almost surely well-defined, and thus $M_{N}$ is well-defined as well irrespective of whether $N<\infty$ or not. Let $Q_{n}=M_{\min \{N, n\}}$ be a stopped version of $\left\{M_{n}\right\}_{n}$. Then an application of Fatou's lemma yields
$$
\mathbb{E}\left[M_{N}\right]=\mathbb{E}\left[\liminf _{n \rightarrow \infty} Q_{n}\right] \leq \liminf _{n \rightarrow \infty} \mathbb{E}\left[Q_{n}\right]=\liminf _{n \rightarrow \infty} \mathbb{E}\left[M_{\min \{N, n\}}\right] \leq 1,
$$
since the stopped martingale $\left\{M_{\min \{N, n\}}\right\}_{n \geq 1}$ is also a martingale. Therefore, by the properties of $M_{n},(12)$ also holds for any random stopping time $N<\infty$.
To complete the proof, we now employ a random stopping time construction as in Abbasi-Yadkori et al. (2011)

We define a random stopping time $N$ by
$$
N=\min \left\{n \geq 1: \sum_{t=1}^{n} \mathrm{KL}_{s_{t}, a_{t}}\left(\thetahr(n), \thetar\right)+\frac{\eta}{2}\left\|\thetar-\thetahr(n)\right\|_{A}^{2}-\frac{\eta}{2}\left\|\thetar\right\|_{A}^{2} \geq \log \left(\frac{C_{A}, n}{\delta}\right)\right\}
$$
with $\min \{\emptyset\}:=\infty$ by convention. We then have
$$
\mathbb{P}\left[\exists n \geq 1, \sum_{t=1}^{n} \mathrm{KL}_{s_{t}, a_{t}}\left(\thetahr(n), \thetar\right)+\frac{\eta}{2}\left\|\thetar-\thetahr(n)\right\|_{\mathrm{A}}^{2}-\frac{\eta}{2}\left\|\thetar\right\|_{\mathrm{A}}^{2} \geq \log \left(\frac{C_{A, n}}{\delta}\right)\right]=\mathbb{P}[N<\infty] \leq \delta,
$$
which concludes the proof of the first part.\\

\textbf{Proof of second part: upper bound on $C_{A, n}$.} First, we have for some $\tilde{\theta} \in\left[\thetahr(n), \theta^{\prime}\right]_{\infty}$ that
\begin{equation}
    \label{eq:KL_reward}
    \mathrm{KL}_{s, a}\left(\thetahr(n), \theta^{\prime}\right)=\frac{1}{2} \sum_{i, j=1}^{d}\left(\theta^{\prime}-\thetahr(n)\right)_{i} \Var_{s,a}^\theta (r) \times \varphi(s,a)^\top A_i^\top B B^\top A_j \varphi(s,a) \left(\theta^{\prime}-\thetahr(n)\right)_{j}
\end{equation}
Now \eqref{eq:KL_reward} implies that
\begin{align*}
    \sum_{t=1}^{n} \mathrm{KL}_{s_{t}, a_{t}}\left(\thetahr(n), \theta^{\prime}\right) &\leq \frac{\beta}{2} \sum_{t=1}^{n} \sum_{i, j=1}^{d}\left(\theta^{\prime}-\thetahr(n)\right)_{i} \varphi\left(s_{t}, a_{t}\right)^{\top} A_{i}^{\top} A_{j} \varphi\left(s_{t}, a_{t}\right)\left(\theta^{\prime}-\thetahr(n)\right)_{j}\\
    &=\frac{\betar}{2}\left\|\theta^{\prime}-\thetahr(n)\right\|_{\sum_{t=1}^{n}}^{2} G_{s_{t}, a_{t}},
\end{align*}
where $\betar:=\lambda_{\max }\left(B B^\top \right) \times \sup _{\theta, s, a}  \Var_{s,a}^\theta(r)$ and $\forall i, j \leq d, \:\left(G_{s, a}\right)_{i, j}:=\varphi(s, a)^{\top} A_{i}^{\top} A_{j} \varphi(s, a) .$ Therefore, we obtain
$$
\begin{aligned}
C_{\mathrm{A}, n} & \leq \frac{\int_{\mathbb{R}^{d}} \exp \left(-\frac{\eta}{2}\left\|\theta^{\prime}\right\|_{\mathrm{A}}^{2}\right) d \theta^{\prime}}{\int_{\mathbb{R}^{d}} \exp \left(-\frac{1}{2}\left\|\theta^{\prime}-\thetahr(n)\right\|_{\left(\betar \sum_{t=1}^{n} G_{s_{t}, a_{t}}+\eta \mathrm{A}\right)}^{2}\right) d \theta^{\prime}} \\
&=\frac{(2 \pi)^{d / 2}}{\det(\eta \mathbb{A})^{1 / 2}} \times \frac{\det\left(\betar \sum_{t=1}^{n} G_{s_{t}, a_{t}}+\eta \mathbb{A}\right)^{1 / 2}}{(2 \pi)^{d / 2}}=\det\left(I+\betar \eta^{-1} \mathbb{A}^{-1} \sum_{t=1}^{n} G_{s_{t}, a_{t}}\right),
\end{aligned}
$$
which completes the proof of the second part.

\end{proof}

\begin{corollary}
\label{cor:Reward_euclidean_confidence_region}
Here also, the theorem implies a euclidean control. With probability at least $1-\delta$ uniformly over $k\in \bN$
\begin{equation*}
    \left\|\thetar-\thetahr(k)\right\|_{\bar{G}_{k}^\texttt{r}}^{2} \le \frac{2}{\alphar} \betar(k,\delta),
\end{equation*}
where $\betar(k,\delta) \eqdef \betar_{(k-1) H}(\delta)=\frac{2}{2} B_{A}^{2}+\log \left(2 C_{A,k}^\texttt{r} / \delta\right)$.
\end{corollary}

\subsection{Gaussian concentration and anti-concentration}

\begin{lemma}[Gaussian concentration, ref. Appendix~A in \cite{abeille2017linear}]
\label{lem:gaussian_concentration} 
Let $ \overline{\xi}_{tk} \sim \mathcal{N}(0, H\nu_k(\delta) \Sigma_{tk}^{-1}) $. For any $ \delta > 0$, with probability $ 1- \delta$
\begin{align}
    \|\overline\xi_{tk}\|_{\Sigma_{tk}} \leq c \sqrt{H d \nu_k(\delta) \log(d/\delta)} 
\end{align}
for some absolute constant $ c$.
\end{lemma}

\begin{lemma}[Gaussian anti-concentration, ref. Appendix~A in \cite{abeille2017linear}] \label{lem:gaussian_anti_concentration} 
    Let $\xi \sim \cN(0,I_d)$, for any $u\in \bR^d$ with $\|u\|=1$, we have:
    \begin{equation*}
        \bP(u^\top\xi\ge 1)\ge \Phi(-1),
    \end{equation*}
    where $\Phi$ is the normal CDF.
\end{lemma}
Thanks to lower bounds on the error function, we have the following bound on the probability of anti-concentration $ \Phi(-1) \ge 1/(4 \sqrt{e \pi})$.

\section{Technical results}\label{app:technical_results}

\subsection{A transportation lemma}\label{app:Transportation}

For any function $f: \mathcal{X} \rightarrow \mathbb{R}$, we define its span as $\mathbb{S}(f):=\max _{x \in \mathcal{X}} f(x)-\min _{x \in \mathcal{X}} f(x) .$ For a probability distribution $P$ supported on the set $\mathcal{X}$, let $\mathbb{E}_{P}[f]:=\mathbb{E}_{P}[f(X)]$ and $\mathbb{V}_{P}[f]:=\mathbb{V}_{P}[f(X)]=\mathbb{E}_{P}\left[f(X)^{2}\right]-$ $\mathbb{E}_{P}[f(X)]^{2}$ denote the mean and variance of the random variable $f(X)$, respectively. We now state the following transportation inequalities, which can be adapted from \cite{boucheron2013concentration} (Lemma 4.18).
\begin{lemma}{(Transportation inequalities)}
\label{lem:transportation}
Assume $f$ is such that $S(f)$ and $\mathbb{V}_{P}[f]$ are finite. Then it holds
\begin{align*}
    \forall Q \ll P, \quad \mathbb{E}_{Q}[f]-\mathbb{E}_{P}[f] &\leq \sqrt{2 \mathbb{V}_{P}[f] \mathrm{KL}(Q, P)}+\frac{2 S(f)}{3} \mathrm{KL}(Q, P) \\
\forall Q \ll P, \quad \mathbb{E}_{P}[f]-\mathbb{E}_{Q}[f] &\leq \sqrt{2 \mathbb{V}_{P}[f] \mathrm{KL}(Q, P)}
\end{align*}
\end{lemma}

\subsection{Bregman divergence}\label{app:Bregman_divergence}

For a Legendre function $F: \mathbb{R}^{d} \rightarrow \mathbb{R}$, the Bregman divergence between $\theta^{\prime}, \theta \in \mathbb{R}^{d}$ associated with $F$ is defined as $B_{F}\left(\theta^{\prime}, \theta\right):=F\left(\theta^{\prime}\right)-F(\theta)-\left(\theta^{\prime}-\theta\right)^{\top} \nabla F(\theta) .$
Now, for any fixed $\theta \in \mathbb{R}^{d}$, we introduce the function
$$
B_{F, \theta}(\lambda):=B_{F}(\theta+\lambda, \lambda)=F(\theta+\lambda)-F(\theta)-\lambda^{\top} \nabla F(\theta) .
$$
It then follows that $B_{F, \theta}$ is a convex function, and we define its dual as
$$
B_{F, \theta}^{\star}(x)=\sup _{\lambda \in \mathbb{R}^{d}}\left(\lambda^{\top} x-B_{F, \theta}(\lambda)\right)
$$
We have for any $\theta, \theta^{\prime} \in \mathbb{R}^{d}$:
\begin{equation}
    \label{eq:Bregman_duality}
    B_{F}\left(\theta^{\prime}, \theta\right)=B_{F, \theta^{\prime}}^{\star}\left(\nabla F(\theta)-\nabla F\left(\theta^{\prime}\right)\right)
\end{equation}
To see this, we observe that
\begin{align*}
    B_{F, \theta^{\prime}}^{\star}&\left(\nabla F(\theta)-\nabla F\left(\theta^{\prime}\right)\right) \\
    &= \sup _{\lambda \in \mathbb{R}^{d}} \lambda^{\top}\left(\nabla F(\theta)-\nabla F\left(\theta^{\prime}\right)\right)-\left[F\left(\theta^{\prime}+\lambda\right)-F\left(\theta^{\prime}\right)-\lambda^{\top} \nabla F\left(\theta^{\prime}\right)\right] \\
    &= \sup _{\lambda \in \mathbb{R}^{d}} \lambda^{\top} \nabla F(\theta)-F\left(\theta^{\prime}+\lambda\right)+F\left(\theta^{\prime}\right) .
\end{align*}

Now an optimal $\lambda$ must satisfy $\nabla F(\theta)=\nabla F\left(\theta^{\prime}+\lambda\right)$. One possible choice is $\lambda=\theta-\theta^{\prime}$. Since, by definition, $F$ is strictly convex, the supremum will indeed be attained at $\lambda=\theta-\theta^{\prime}.$ Plugin-in this value, we obtain
$$
B_{F, \theta^{\prime}}^{\star}\left(\nabla F(\theta)-\nabla F\left(\theta^{\prime}\right)\right)=\left(\theta-\theta^{\prime}\right)^{\top} \nabla F(\theta)-F(\theta)+F\left(\theta^{\prime}\right)=B_{F}\left(\theta^{\prime}, \theta\right) .
$$
Note that \eqref{eq:Bregman_duality} holds for any convex function $F$. Only difference is that, in this case, $B_{F}(\cdot, \cdot)$ will not correspond to the Bregman divergence.

\subsection{Properties of the bilinear exponential family} \label{app:properties_exp_fam}

In this section, we detail some useful results related to exponential families in our model.

\subsubsection{Derivatives}\label{app:derivatives}

\begin{lemma}{(Gradients)} \label{lem:Exp_Gradients}
    We provide the derivatives of the log-partitions in closed form. As usual with exponential families, these are intimately linked to moments of the random variable. We have:
        \begin{equation*}
            \left(\nabla_{i} \pz_{s, a}\right)(\theta) = \mathbb{E}_{s, a}^{\theta}\left[\psi\left(s^{\prime}\right)\right]^{\top} A_{i} \varphi(s, a).
        \end{equation*}
    And
        \begin{equation*}
            \left(\nabla_{i} \rz_{s, a}\right)(\theta) =\mathbb{E}_{s, a}^{\theta}\left[r\right]\: B^{\top} A_{i} \varphi(s, a).
        \end{equation*}
\end{lemma}

\begin{proof}
    We prove the lemma as follows
    \begin{align*}
    \left(\nabla_{i} \pz_{s, a}\right)(\theta) &=\int_{\mathcal{S}} \psi\left(s^{\prime}\right)^{\top} A_{i} \varphi(s, a) \frac{ \exp \left(\sum_{i=1}^{d} \theta_{i} \psi\left(s^{\prime}\right)^{\top} A_{i} \varphi(s, a)\right)}{\int_{\mathcal{S}} \exp \left(\sum_{i=1}^{d} \theta_{t} \psi\left(s^{\prime}\right)^{\top} A_{i} \varphi(s, a)\right) d s^{\prime}} d s^{\prime} \\
    &=\mathbb{E}_{s, a}^{\theta}\left[\psi\left(s^{\prime}\right)\right]^{\top} A_{i} \varphi(s, a)\\
    \left(\nabla_{i} \rz_{s, a}\right)(\theta) &=\int_{\mathcal{S}} r B^{\top} A_{i} \varphi(s, a) \frac{ \exp \left(r \sum_{i=1}^{d} \theta_i B^{\top} A_{i} \varphi(s, a)\right)}{\int_{\mathcal{S}} \exp \left(r \sum_{i=1}^{d} \theta_i B^{\top} A_{i} \varphi(s, a)\right) d r} d r \\
    &=\mathbb{E}_{s, a}^{\theta}\left[r\right]\: B^{\top} A_{i} \varphi(s, a)
    \end{align*}
\end{proof}

\begin{lemma}{(Hessians)}\label{lem:Exp_Hessians}
    The entries of the Hessians of the log partition functions are given by
        \begin{equation*}
            \left(\nabla_{i, j}^{2} \pz_{s, a}\right)(\theta) = \varphi(s, a)^{\top} A_{i}^{\top} \mathbb{C}_{s, a}^{\theta}\left[\psi\left(s^{\prime}\right)\right] A_{j} \varphi(s, a),
        \end{equation*}
    where $\mathbb{C}_{s, a}^{\theta}\left[\psi\left(s^{\prime}\right)\right] \eqdef \mathbb{E}_{s, a}^{\theta}\left[\psi\left(s^{\prime}\right) \psi\left(s^{\prime}\right)^{\top}\right]-\mathbb{E}_{s, a}^{\theta}\left[\psi\left(s^{\prime}\right)\right] \mathbb{E}_{s, a}^{\theta}\left[\psi\left(s^{\prime}\right)^{\top}\right]$. 
    
    Similarly,
    \begin{equation*}
        \left(\nabla_{i, j}^{2} \rz_{s, a}\right)(\theta) = \Var_{s,a}^\theta (r) \times \varphi(s,a)^\top A_i^\top B B^\top A_j \varphi(s,a),
    \end{equation*}
    where $\Var_{s,a}^\theta (r) \eqdef \left(\mathbb{E}_{s, a}^{\theta}\left[r^2\right]-\mathbb{E}_{s, a}^{\theta}\left[r\right]^2\right)$ is the variance of the reward under $\theta$.
\end{lemma}

\begin{proof}
    We prove these formulas by differentiating under the integral sign.
    \begin{align*}
    \left(\nabla_{i, j}^{2} \pz_{s, a}\right)(\theta) &=  \int_{\mathcal{S}} \psi\left(s^{\prime}\right)^{\top} A_{i} \varphi(s, a) \psi\left(s^{\prime}\right)^{\top} A_{j} \varphi(s, a) \frac{\exp \left(\sum_{i=1}^{d} \theta_{i} \psi\left(s^{\prime}\right)^{\top} A_{i} \varphi(s, a)\right)}{\int_{\mathcal{S}}  \exp \left(\sum_{i=1}^{d} \theta_{i} \psi\left(s^{\prime}\right)^{\top} A_{i} \varphi(s, a)\right) d s^{\prime}} d s^{\prime} \\
    &- \int_{\mathcal{S}} \psi\left(s^{\prime}\right)^{\top} A_{i} \varphi(s, a) \frac{ \exp \left(\sum_{i=1}^{d} \theta_{i} \psi\left(s^{\prime}\right)^{\top} A_{i} \varphi(s, a)\right)}{\int_{\mathcal{S}} \exp \left(\sum_{i=1}^{d} \theta_{i} \psi\left(s^{\prime}\right)^{\top} A_{i} \varphi(s, a)\right) d s^{\prime}} d s^{\prime}\left(\nabla_{j} Z_{s, a}\right)(\theta) \\
    &= \mathbb{E}_{s, a}^{\theta}\left[\psi\left(s^{\prime}\right)^{\top} A_{i} \varphi(s, a) \psi\left(s^{\prime}\right)^{\top} A_{j} \varphi(s, a)\right] \\
    &-\mathbb{E}_{s, a}^{\theta}\left[\psi\left(s^{\prime}\right)^{\top} A_{i} \varphi(s, a)\right] \mathbb{E}_{s, a}^{\theta}\left[\psi\left(s^{\prime}\right)^{\top} A_{j} \varphi(s, a)\right] \\
    =& \varphi(s, a)^{\top} A_{i}^{\top}\left(\mathbb{E}_{s, a}^{\theta}\left[\psi\left(s^{\prime}\right) \psi\left(s^{\prime}\right)^{\top}\right]-\mathbb{E}_{s, a}^{\theta}\left[\psi\left(s^{\prime}\right)\right] \mathbb{E}_{s, a}^{\theta}\left[\psi\left(s^{\prime}\right)^{\top}\right]\right) A_{j} \varphi(s, a) \\
    =& \varphi(s, a)^{\top} A_{i}^{\top} \mathbb{C}_{s, a}^{\theta}\left[\psi\left(s^{\prime}\right)\right] A_{j} \varphi(s, a),
    \end{align*}
    
    where we introduce in the last line the $p \times p$ covariance matrix given by
    $$ \mathbb{C}_{s, a}^{\theta}\left[\psi\left(s^{\prime}\right)\right]=\mathbb{E}_{s, a}^{\theta}\left[\psi\left(s^{\prime}\right) \psi\left(s^{\prime}\right)^{\top}\right]-\mathbb{E}_{s, a}^{\theta}\left[\psi\left(s^{\prime}\right)\right] \mathbb{E}_{s, a}^{\theta}\left[\psi\left(s^{\prime}\right)^{\top}\right]$$
    
    The proof of the form of the Hessian for the reward partition function follows the same steps as above.    
\end{proof}

\begin{lemma}{(KL Divergences)}\label{lem:Exp_KL_div}
    For any two $\theta, \theta^{\prime}$ and for some pair $(s, a)$,
    \begin{equation*}
        \exists \tilde{\theta} \in\left[\theta, \theta^{\prime}\right]_{\infty}, \quad \mathrm{KL}\left(P_{\theta}^\texttt{p} (\cdot \mid s, a), P_{\theta^{\prime}}^\texttt{p} (\cdot \mid s, a)\right) = \frac{1}{2}\left(\theta-\theta^{\prime}\right)^{\top}\left(\nabla^{2} \pz_{s, a}\right)(\tilde{\theta})\left(\theta-\theta^{\prime}\right),
    \end{equation*}
    where $\left[\theta, \theta^{\prime}\right]_{\infty}$ denotes the $d$-dimensional hypercube joining $\theta$ to $\theta^{\prime}$.
     
    Similarly
    \begin{equation*}
        \exists \tilde{\theta} \in\left[\theta, \theta^{\prime}\right]_{\infty}, \quad \mathrm{KL}\left(P_{\theta}^\texttt{r}(\cdot \mid s, a), P_{\theta^{\prime}}^\texttt{r} (\cdot \mid s, a)\right) = \frac{1}{2}\left(\theta-\theta^{\prime}\right)^{\top}\left(\nabla^{2} \rz_{s, a}\right)(\tilde{\theta})\left(\theta-\theta^{\prime}\right).
    \end{equation*}
\end{lemma}

\begin{proof}
    We start by writing:
    \begin{equation*}
        \log \left(\frac{P_{\theta}^\texttt{p} \left(s^{\prime} \mid s, a\right)}{P_{\theta^{\prime}}^\texttt{p} \left(s^{\prime} \mid s, a\right)}\right) =\sum_{i=1}^{d}\left(\theta_{i}-\theta_{i}^{\prime}\right) \psi\left(s^{\prime}\right)^{\top} A_{i} \varphi(s, a)-\pz_{s, a}(\theta)+\pz_{s, a}\left(\theta^{\prime}\right),
    \end{equation*}
    then
    \begin{align*}
        \mathrm{KL}\left(P_{\theta}^\texttt{p} (\cdot \mid s, a), P_{\theta^{\prime}}^\texttt{p} (\cdot \mid s, a)\right) &=\sum_{i=1}^{d}\left(\theta_{i}-\theta_{i}^{\prime}\right) \mathbb{E}_{s, a}^{\theta}\left[\psi\left(s^{\prime}\right)\right]^{\top} A_{i} \varphi(s, a)-\pz_{s, a}(\theta)+\pz_{s, a}\left(\theta^{\prime}\right) \\
        &=\frac{1}{2}\left(\theta-\theta^{\prime}\right)^{\top}\left(\nabla^{2} \pz_{s, a}\right)(\tilde{\theta})\left(\theta-\theta^{\prime}\right),
    \end{align*}
    where in the last line, we used, by a Taylor expansion, that $Z_{s,a} \left( \theta^{\prime} \right) = Z_{s, a} (\theta)+\left(\nabla Z_{s, a}(\theta)\right)^{\top}\left(\theta^{\prime}-\theta\right)+\frac{1}{2}(\theta-$ $\left.\theta^{\prime}\right)^{\top}\left(\nabla^{2} Z_{s, a}(\tilde{\theta})\right)\left(\theta-\theta^{\prime}\right)$ for some $\tilde{\theta} \in\left[\theta, \theta^{\prime}\right]_{\infty}$.
    
    The proof of the form of the KL divergence for the reward follows the same steps as above.
\end{proof}

\subsubsection{A transportation lemma for rewards} \label{subsubsection:Transportation_reward}

\begin{lemma}\label{lem:Reward_Transportation}
    We provide a closed-form formula for the difference of expected rewards under two distinct parameters:
    \begin{equation*}
        \exists \theta_3 \in [\theta_1, \theta_2], \qquad \mathbb{E}_{s, a}^{\theta_1}\left[r\right] = \mathbb{E}_{s, a}^{\theta_2}\left[r\right] + \frac{ \Var_{s,a}^{\theta_3}  (r)}{ 2 } B^\top M_{\theta_1 - \theta_2}\varphi(s,a)
    \end{equation*}
\end{lemma}

\begin{proof}
    Let's recall the gradient of the reward log partition function:
    \begin{equation*}
        \left(\nabla_{i} \rz_{s, a}\right)(\thetar) =\mathbb{E}_{s, a}^{\thetar}\left[r\right]\: B^{\top} A_{i} \varphi(s, a)
    \end{equation*}
    then for all $\thetarp$ we have:
    \begin{equation*}
        \mathbb{E}_{s, a}^{\thetar}\left[r\right] = \frac{1}{ B^{\top} M_{\thetarp} \varphi(s, a)}  \nabla_{i} \rz_{s, a}(\thetar)^\top \thetarp
    \end{equation*}
    Let $\theta_1 , \theta_2 \in \bR^d$, using Taylor-Cauchy's formula there exists $\theta_3 \in [\theta_1, \theta_2]$ such that:
    \begin{align*}
        \mathbb{E}_{s, a}^{\theta_1}\left[r\right] = \mathbb{E}_{s, a}^{\theta_2}\left[r\right] + \frac{1}{ 2 B^{\top} M_{\thetarp} \varphi(s, a)} (\theta_1 - \theta_2)^\top \nabla^2 \rz_{s, a}(\theta_3)^\top \thetarp
    \end{align*}
    We know that $\left(\nabla_{i, j}^{2} \rz_{s, a}\right)(\theta) = \Var_{s,a}^\theta (r) \times \varphi(s,a)^\top A_i^\top B B^\top A_j \varphi(s,a)$, choosing $\thetarp = \theta_1 - \theta_2$ we find:
    \begin{equation*}
        \mathbb{E}_{s, a}^{\theta_1}\left[r\right] = \mathbb{E}_{s, a}^{\theta_2}\left[r\right] + \frac{ \Var_{s,a}^{\theta_3}  (r)}{ 2 } B^\top M_{\theta_1 - \theta_2}\varphi(s,a).
    \end{equation*}
\end{proof}

\subsection{Elliptical potentials and elliptical lemma} \label{app:elliptical_potentials}

\subsubsection{Elliptical lemma}\label{app:elliptical_lemma}

Here we show a lemma that is popular for regret control in linear MDPs and linear Bandits. 

First, consider the notations: $G_{s, a} :=(\varphi(s, a)^{\top} A_{i}^{\top} A_{j} \varphi(s, a))_{1\le i,j \le d} \:$, $\quad \bar{G}_{n}^\texttt{e} \equiv \bar{G}_{(k-1) H}^\texttt{e}:=G_{n}+(\alpha^\texttt{e})^{-1} \eta A \:$, and $G_{n} \equiv G_{(k-1) H}:=\sum_{\tau=1}^{k-1} \sum_{h=1}^{H} G_{s_{s}^{\tau}, a_{h}^{\tau}}$. Where $\texttt{e}$ represents either $\texttt{r}$ or $\texttt{p}$, we omit the superscript $\texttt{e}$ w.l.o.g in the rest of this section.

\begin{lemma}{(Elliptical lemma and variant for bounded potentials)} \label{lem:elliptical}
    Let $c\in \bR^+$, we can bound the sum of feature norms as follows
    \begin{equation*}
        \sum_{t=1}^{T} \min\{c, \sum_{h=1}^{H}\left\|\bar{G}_n^{-1/2} G_{s,a} \bar{G}_n^{-1/2}\right\|\} \le \frac{c}{\log(1+c)} d \log \left(1+\alpha \eta^{-1} B_{\varphi, \mathbb{A}} n\right).
    \end{equation*}
    where $B_{\varphi, \mathbb{A}}:=\sup _{s, a}\left\|\mathbb{A}^{-1} G_{s, a}\right\|$.
    
    Further, we have
    \begin{equation*}
        \sum_{t=1}^{T} \sum_{h=1}^{H}\left\|\bar{G}_n^{-1/2} G_{s,a} \bar{G}_n^{-1/2}\right\| \le 2d \log \left(1+\alpha \eta^{-1} B_{\varphi, \mathbb{A}} n\right) + \frac{3 d H}{\log (2)} \log \left(1+\frac{\alpha \|A\|_2^2 B_{\varphi, \mathbb{A}}^{2}}{\eta \log (2)}\right) 
    \end{equation*}
\end{lemma}

\begin{proof}
First we have
\begin{align*}
    \|\bar{G}_n^{-1/2} G_{s,a} \bar{G}_n^{-1/2}\| &= \sqrt{\tr(\bar{G}_n^{-1/2} G_{s,a} \bar{G}_n^{-1/2}\bar{G}_n^{-1/2} G_{s,a} \bar{G}_n^{-1/2})}\\
    &\le \tr(\bar{G}_n^{-1/2} G_{s,a} \bar{G}_n^{-1/2}) = \tr(\bar{G}_n^{-1} G_{s,a}) = \tr(\boldsymbol{a}_h^\top \bar{G}_n^{-1}\boldsymbol{a}_h)
\end{align*}
the last line is because $G_{s,a} = \boldsymbol{a}_h\boldsymbol{a}_h^\top$, where $\boldsymbol{a}_h = (A_i\varphi(s_h,a_h))_{i\in[d]}$.

\textbf{First result.} Consider $h\in[H]$, denote $(\lambda_{h,i}){i\in[d]}$ the eigenvalues of $\boldsymbol{a}_h^\top \bar{G}_n^{-1}\boldsymbol{a}_h$. $\bar{G}_n$ is positive definite hence $\lambda_{h,i} > 0, \forall h,i$, then
\begin{align*}
    \min\{c,\sum_{h=1}^H \tr(\boldsymbol{a}_h^\top \bar{G}_n^{-1}\boldsymbol{a}_h)\} &= \min\{c,\sum_{h=1}^H \sum_{i=1}^d \lambda_{h,i}\}\\
    &\le \frac{c}{\log(1+c)} \sum_{h=1}^H \sum_{i=1}^d \log(1+\lambda_{h,i}) \tag{$\log$ is concave}\\
    &\le \frac{c}{\log(1+c)} \sum_{h=1}^H \log(\prod_{i=1}^d 1+\lambda_{h,i}) = \frac{c}{\log(1+c)} \sum_{h=1}^H \log \det(I + \boldsymbol{a}_h^\top \bar{G}_n^{-1}\boldsymbol{a}_h)\\
    &\le \frac{c}{\log(1+c)}\log\left(\frac{\det(\bar{G}_n + \sum_{h=1}^H G_{s_h,a_h})}{\det(\bar{G}_n)}\right)
\end{align*}
where the last line follows from the matrix determinant lemma:
\begin{equation*}
    \det\left(\bar{G}_n + \boldsymbol{a}_h \boldsymbol{a}_h^\top\right) = \det(I + \boldsymbol{a}_h^\top \bar{G}_n^{-1}\boldsymbol{a}_h)\det(\bar{G}_n)
\end{equation*}
Therefore:
\begin{equation*}
    \sum_{t=1}^{T} \min\{c, \sum_{h=1}^{H}\left\|\bar{G}_{n}^{-1} G_{s_{h}^{t}, a_{h}^{t}}\right\|\} \le \frac{c}{\log(1+c)} \sum_{t=1}^{T} \log \frac{\det\left(\bar{G}_{n+H}\right)}{\det\left(\bar{G}_{n}\right)},
\end{equation*}

We can now control the R.H.S. of the above equation, as
\begin{align*}
    \sum_{t=1}^{T} \log& \frac{\det\left(\bar{G}_{n+H}\right)}{\det\left(\bar{G}_{n}\right)} = \sum_{t=1}^{T} \log \frac{\det\left(\bar{G}_{t H}\right)}{\det\left(\bar{G}_{(t-1) H}\right)}=\log \frac{\det\left(\bar{G}_{T H}\right)}{\det\left(\bar{G}_{0}\right)}\\
    &= \log \frac{\det\left(\bar{G}_{N}\right)}{\det\left((\alphap)^{-1} \eta \mathbb{A}\right)}=\log \det\left(I+\alpha \eta^{-1} \mathrm{~A}^{-1} G_{N}\right)\\
    &\le d \log\left(1+\frac{\alphap \eta^{-1}}{d} \operatorname{tr}\left(\mathbb{A}^{-1} G_{n}\right)\right) \tag{Trace-determinant (or AM-GM) inequality}\\
    &\le d \log \left(1+\alphap \eta^{-1} B_{\varphi, \mathbb{A}} n\right)
\end{align*}
This concludes the proof of the first result.

\textbf{Second result.}
First, we have $\sup _{s, a}\left\|G_{s, a}\right\|_2 \le \|A\|_2 B_{\varphi, \mathbb{A}}$.

Fix an episode $k\in [K], n = (k-1)H$, using Lemma~\ref{lem:worst_case_elliptical}, we know that the number of times $h\in [h]$ such that $\: \left\|\bar{G}_{n}^{-1} G_{s_h, a_h}\right\|\ge 1$ is smaller than $\frac{3 d}{\log (2)} \log \left(1+\frac{\alpha (\|A\|_2 B_{\varphi, \mathbb{A}})^{2}}{\eta \log (2)}\right)$. Let us call $\mathcal{T}_{k}:=\{h\in [H] \left\|\bar{G}_{(k-1)h}^{-1} G_{s_h, a_h}\right\|\le 1\}$, then
\begin{align*}
    \sum_{t=1}^{T} \sum_{h=1}^{H}\left\|\bar{G}_{n}^{-1} G_{s_{h}^{t}, a_{h}^{t}}\right\| \le \frac{3 d}{\log (2)} \log \left(1+\frac{\alpha \|A\|_2^2 B_{\varphi, \mathbb{A}}^{2}}{\eta \log (2)}\right) + \sum_{h\in \mathcal{T}_{k}} \min\{1, \left\|\bar{G}_{n}^{-1} G_{s_{h}^{t}, a_{h}^{t}}\right\|\}
\end{align*}
the sum of the right hand side is similar to the first result. Although the sum is not contiguous, the previous bound holds since if $h_1 < h_2, \det(\bar{G}_{n+h_1})\le \det(\bar{G}_{n+h_2})$, this concludes the proof.
\end{proof}

\begin{remark}
\label{rk:elliptical_norm_of_a}
We can also write from the lemma in terms of $\left\|(A_i\varphi(\tilde{s}_h,\pi(\tilde{s}_h)))_{1\le i\le d}\right\|_{(\rG)^{-1}}$ by skipping the norm upper bound at the beginning of the proof:
\begin{equation*}
    \sum_{t=1}^{T} \min\{c, \sum_{h=1}^{H}\left\|(A_i\varphi(\tilde{s}_h,\pi(\tilde{s}_h)))_{1\le i\le d}\right\|_{(\rG)^{-1}}\} \le \frac{c}{\log(1+c)} d \log \left(1+\alpha \eta^{-1} B_{\varphi, \mathbb{A}} n\right).
\end{equation*}
and
\begin{align*}
    \sum_{t=1}^{T} \sum_{h=1}^{H}\left\|(A_i\varphi(\tilde{s}_h,\pi(\tilde{s}_h)))_{1\le i\le d}\right\|_{(\rG)^{-1}} \le& 2d \log \left(1+\alpha \eta^{-1} B_{\varphi, \mathbb{A}} n\right)\\
    &+ \frac{3 d H}{\log (2)} \log \left(1+\frac{\alpha \|A\|_2^2 B_{\varphi, \mathbb{A}}^{2}}{\eta \log (2)}\right) 
\end{align*}
\end{remark}

\subsubsection{Elliptical potentials: finite number of large feature norms (contribution)}\label{app:elliptical_finite_amnt_big_intervals}

\begin{lemma}{(Worst case elliptical potentials, adaptation of Exercise 19.3~\cite{lattimore2020bandit} for matrices)} \label{lem:worst_case_elliptical}
    Let $V_{0}=\lambda I$ and $a_{1}, \ldots, a_{n} \in \mathbb{R}^{d\times p}$ be a sequence of matrices with $\left\|a_{t}\right\|_{2} \leq L$ for all $t \in[n]$. Let $V_{t}=V_{0}+\sum_{s=1}^{t} a_{s} a_{s}^{\top}$, then
    \begin{equation*}
        \left|\{t \in \mathbb{N}^*, \left\|a_{t}\right\|_{V_{t-1}^{-1}} \geq 1\}\right| \le \frac{3 d}{\log (2)} \log \left(1+\frac{L^{2}}{\lambda \log (2)}\right)
    \end{equation*}
\end{lemma}

\begin{proof}
    Let $\mathcal{T}$ be the set of rounds $t$ when $\left\|a_{t}\right\|_{V_{t-1}^{-1}} \geq 1$ and $G_{t}=V_{0}+\sum_{s=1}^{t} \mathbb{I}_{\mathcal{T}}(s) a_{s} a_{s}^{\top}$. Then
    \begin{align*}
        \left(\frac{d \lambda+|\mathcal{T}| L^{2}}{d}\right)^{d} & \geq\left(\frac{\operatorname{trace}\left(G_{n}\right)}{d}\right)^{d} \\
        & \geq \det\left(G_{n}\right) \tag{Trace-determinant inequality}\\
        &=\operatorname{det}\left(V_{0}\right) \prod_{t \in T}\left(1+\left\|a_{t}\right\|_{G_{t-1}^{-1}}^{2}\right) \\
        & \geq \operatorname{det}\left(V_{0}\right) \prod_{t \in T}\left(1+\left\|a_{t}\right\|_{V_{t-1}^{-1}}^{2}\right) \\
        & \geq \lambda^{d} 2^{|\mathcal{T}|}
    \end{align*}
    where the third line follows from the matrix determinant lemma:
    \begin{equation*}
        \det\left(\bar{G}_n + \boldsymbol{a}_h \boldsymbol{a}_h^\top\right) = \det(I + \boldsymbol{a}_h^\top \bar{G}_n^{-1}\boldsymbol{a}_h)\det(\bar{G}_n).
    \end{equation*}
    Rearranging and taking the logarithm shows that
    $$
    |\mathcal{T}| \leq \frac{d}{\log (2)} \log \left(1+\frac{|\mathcal{T}| L^{2}}{d \lambda}\right)
    $$
    Abbreviate $x=d / \log (2)$ and $y=L^{2} / d \lambda$, which are both positive. Then
    $$
    x \log (1+y(3 x \log (1+x y))) \leq x \log \left(1+3 x^{2} y^{2}\right) \leq x \log (1+x y)^{3}=3 x \log (1+x y).
    $$
    Since $z-x \log (1+y z)$ is decreasing for $z \geq 3 x \log (1+x y)$ it follows that
    $$
    |\mathcal{T}| \leq 3 x \log (1+x y)=\frac{3 d}{\log (2)} \log \left(1+\frac{L^{2}}{\lambda \log (2)}\right).
    $$
\end{proof}

\section{Tractable planning with random Fourier transform}\label{app:RFT}

\textbf{A Primer on random Fourier transforms.} We start by defining the Random Fourier Transform and its most relevant property. 
Let us consider the transition model of Equation~\eqref{def:transition_model}, we have
\begin{align*}
    \bP(s' \mid s,a, \theta) = \exp\left(\psi(s') M_\theta \varphi(s,a) - Z_{\theta}(s,a)\right) = \bE_{p(w,b)}\left[f\left(\psi(s'),w,b\right) f\left(M_\theta \varphi(s,a),w,b\right)\right],
\end{align*}
where $f\left(x,w,b\right) = \sqrt{2}\cos(w^\top x +b)$ are the random Fourier bases. $p(w,b) = \cN(0,\sigma^{-2} I) \times \mathcal{U}([0,2\pi])$, such that $\cN$ is the Gaussian distribution, $\mathcal{U}$ is the Uniform distribution, and $p(w,b)$ is a coupling among them. 

Notice that this provides an alternative approach to decompose the transition kernel and obtain linearity of the value function. Moreover, since $\forall x,w \in \bR^d , b \in \bR, |f(x,w,b)| \le \sqrt{2}$, we can use Hoeffding's inequality to prove that a Monte-Carlo approximation of $\bP(s' \mid s,a, \theta)$ using $N$ sample pairs of $(w,b)$ guarantees an error smaller than $\epsilon$ with probability at least $1-2\exp(-N\epsilon^2 /4)$. \cite{rahimi2007random} proves a stronger result: it provides an algorithm approximating the Gaussian kernel for which the following uniform convergence bound holds.

\begin{lemma}\label{lem:RFT_uniform_convergence}
    Let $\mathcal{M}$ be a compact subset of $\mathcal{R}^{p}$ with diameter $\operatorname{diam}(\mathcal{M})$. Then, using the explicit mapping $\mathbf{z}$ defined in Algorithm 1 in \cite{rahimi2007random} with $N$ samples, we have
    \begin{equation*}
        \operatorname{Pr}\left[\sup _{x, y \in \mathcal{M}}\left|\mathbf{z}(\mathbf{x})^{\prime} \mathbf{z}(\mathbf{y})-k(\mathbf{y}, \mathbf{x})\right| \geq \epsilon\right] \leq 2^{8}\left(\frac{\sigma_{p} \operatorname{diam}(\mathcal{M})}{\epsilon}\right)^{2} \exp \left(-\frac{N \epsilon^{2}}{4(p+2)}\right)
    \end{equation*}
    where $\sigma_{p}^{2} \equiv E_{p}\left[\omega^{\prime} \omega\right]$ is the second moment of the Fourier transform of $k$.
\end{lemma}
Further, it implies that if $N=\Omega\left(\frac{p}{\epsilon^{2}} \log \frac{\sigma_{p} \operatorname{diam}(\mathcal{M})}{\epsilon}\right)$, then $\sup _{x, y \in \mathcal{M}}\left|\mathbf{z}(\mathbf{x})^{\prime} \mathbf{z}(\mathbf{y})-k(\mathbf{y}, \mathbf{x})\right| \leq \epsilon$ with constant probability.

\textbf{Application to planning in \algo.} Since our regret analysis is done under the high probability event of bounded estimation parameters, we know that the spaces of $\psi(s')$ and $M_\theta \varphi(s,a)$ are bounded and the diameter depends on the dimensions. We abstain from explicating the exact diameter as it only influences the number of samples logarithmically. Using $N \approx p/\epsilon^{2}$ samples, we can construct a uniform $\epsilon$-approximation of $\bP(s' \mid s,a, \theta)$.

Let's call $\hat{V}_h$ the estimated value function using Algorithm~\ref{alg_2:planning} with the above approximation of transition. Here, we elucidate the span of this estimation of value function. First we have:
\begin{equation*}
    \hat{V}_H^\pi - V_H^\pi = \int_{s'} (\hat{P}-P)(s'\mid s,a) r(s',\pi(s')) \dd s'
    \le \epsilon d H^{3/2}
\end{equation*}
Here, we use the facts that $\mathbb{S}\left(V_{\hat{\theta}, \tilde{\theta}^{\mathrm{x}}, h}\right) \leq d H^{3 / 2}$ (\cf Section~\ref{app:learning_error}) and the error in approximating $P$ is bounded by $\epsilon$, i.e. $\sup_{s',s,a} |(\hat{P}-P)(s'|s,a)| \leq \epsilon$.

Assume that at step $h+1$, we have $\hat{V}_{h+1}^\pi - V_{h+1}^\pi \le \sum_{j=1}^{h+1} \epsilon^j \alpha_{h+1,j}$. Then, we obtain
\begin{align*}
    \hat{V}_{h}^\pi - V_{h}^\pi &\le \int_{s'} (\hat{P}-P)(s'\mid s,a) \hat{V}_{h+1}^\pi(s') \dd s' + \int_{s'} P(s'\mid s,a)(\hat{V}_{h+1}^\pi - V_{h+1}^\pi)(s') \dd s' \\
    &= \int_{s'} (\hat{P}-P)(s'\mid s,a) (V_{h+1}^\pi + \hat{V}_{h+1}^\pi - V_{h+1}^\pi) \dd s' + \int_{s'} P(s'\mid s,a)(\hat{V}_{h+1}^\pi - V_{h+1}^\pi)(s') \dd s'\\
    &\le \epsilon(d H^{3/2} + \sum_{j=1}^{h+1} \epsilon^j \alpha_{h+1,j}) + \sum_{j=1}^{h+1} \epsilon^j \alpha_{h+1,j}\\
    &\le \epsilon (d H^{3/2} + \alpha_{h+1,1}) + \sum_{j=2}^{h+1} \epsilon^j (\alpha_{h+1,j-1} + \alpha_{h+1,j}) + \epsilon^{h+2} \alpha_{h+1,h+1}
\end{align*}
Using the fact that $\alpha_{1,1}=d H^{3/2}$ and with a proper induction, we find that:
\begin{equation*}
    \hat{V}_{1}^\pi - V_{1}^\pi \le \epsilon d H^{5/2}  \frac{1-\epsilon^{H-h}}{1-\epsilon} \underset{H \rightarrow \infty}{\leq} \epsilon d H^{5/2}
\end{equation*}

This concludes the proof of the arguments provided in §~Planning of Section~\ref{sec:planning}. This means that the extra regret due to planning with the approximation by RFT features is of order $\bigO(\epsilon d H^{5/2} K)$. By choosing an $\epsilon$ of order $1/(H \sqrt{K})$, we deduce that approximating the probability kernel with $\bigO(p H^2 K)$ samples induces a tractable planning procedure without harming the regret.

\begin{remark}
    The reader might be tempted to combine the finite approximation using RFT with algorithms from the linear reinforcement learning literature \cite{jin2020provably}. However, note that the dimensionality of the linear space induced by RFT is polynomial in $H$ and $K$. Consequently, applying algorithms designed with the assumption of linear value function incurs a linear regret.
\end{remark}

\section{Tractable Maximum Likelihood estimation}\label{app:MLe}

The maximum likelihood estimation is explicit for simple distributions like the Gaussian \cite{rogers1977explicit} and for Linearly controlled dynamical systems. But it requires integral approximations for generic transitions. However, we believe that this estimation problem is far simpler than the planning problem since the latter traditionally involves approximating an integral for all state-action pairs. 

Different approximation techniques have been used in literature to handle the penalized ML estimation. For instance, \textit{Integral Approximation} techniques are well studied for this problem. Indeed, \cite{neal2001annealed} proposes to handle the ML estimation using simulated annealing, a method that starts from a tractable distribution and updates it to resemble the distribution at hand. \cite{vembu2012probabilistic} proposes \textit{MCMC techniques} for approximating the partition function. \cite{carreira2005contrastive} shows that optimizing a different objective, called the \textit{contrastive divergence} leads to a good approximation of the ML. Another line of work is related to \textit{Score matching}, a technique that avoids approximating the partition function and is well studied in literature, see \cite{jorgensen1983maximum}. More recently, \cite{li2021exponential} proposed an adaptation of this technique to the exact setting we consider. The latter shows that under certain conditions, that we are unable to verify, the estimation can be solved in $\mathcal{O}(d^3)$ time. Furthermore, in the case of \textit{Bounded distribution support and natural parameter}, \cite{shah2021computationally} shows that for a minimally represented $k$-parameter Exponential family, an $\alpha$-approximation of the ML can be derived in $\mathcal{O}(\operatorname{poly}(k/\alpha))$ time. The latter assumes a specific definition of compactness of the representation as well as knowledge of the support and shows how to re-parameterize the density to a specific class of exponential families that are easier to study. Finally, \cite{dai2019kernel} studies \textit{exponential families such that the natural parameter belongs to an RKHS}, it proposes a method that improves over score matching in time and in memory complexity.




\end{document}